\documentclass[letterpaper]{article}
\usepackage[numbers]{natbib}

\usepackage{graphicx} 
\usepackage{tikz-cd}
\usepackage{amssymb}
\usepackage{dsfont}
\usepackage{mleftright}
\usepackage{ntheorem}
\usepackage{mathtools}
\usepackage{algorithm, algpseudocodex}
\usepackage[colorlinks=true, allcolors=blue]{hyperref}
\usepackage{cleveref}
\usepackage{multirow}
\usepackage{amsmath,amsfonts,graphicx}
\usepackage{xfrac}
\usepackage{tikz}
\usepackage{enumitem}
\usepackage{booktabs}
\usepackage{thm-restate}

\newtheorem{theorem}{Theorem}[section]
\newtheorem{lemma}[theorem]{Lemma}

\newtheorem{corollary}[theorem]{Corollary}
\newtheorem{definition}[theorem]{Definition}

\newenvironment{proof}{{\bf Proof:}}{\hfill\rule{2mm}{2mm}}
\newenvironment{proofsketch}[1][]
  {\noindent\textit{Proof Sketch\ifx&#1&\else\ #1\fi.} }
  {\hfill$\square$}

\newcommand\E{\mathbb{E}}

\newcommand{\A}{\mathcal{A}}

\newcommand{\D}{\mathcal{D}}

\newcommand{\R}{\mathcal{R}}

\newcommand{\supp}{\mathrm{Supp}}
\newcommand{\piell}{\pi_\ell}
\newcommand{\phase}{m}

\DeclareMathOperator*{\argmin}{arg\,min}

\newcommand{\rr}{\mathbb{R}}

\newcommand{\calF}{\mathcal{F}}
\newcommand{\bbP}{\mathbb{P}}
\newcommand{\bbE}{\mathbb{E}}

\newcommand\inprod[1]{\langle #1 \rangle}

\newcommand{\curly}[1]{ {\left\{ #1 \right\}}}
\newcommand{\roundy}[1]{ {\left( #1 \right)}}
\newcommand{\squary}[1]{ {\left[ #1 \right]}}
\newcommand{\abs}[1]{ {\left | #1 \right |}}
\newcommand{\norm}[1]{\left \| #1 \right \|}
\newcommand{\ceil}[1]{\left \lceil #1 \right \rceil}
\newcommand{\thetastar}{\theta^\star}
\newcommand{\astar}{a^\star}

\newcommand{\calA}{{\mathcal{A}}}

\newcommand{\calD}{{\mathcal{D}}}
\newcommand{\calE}{{\mathcal{E}}}

\newcommand{\calP}{{\mathcal{P}}}
\newcommand{\calQ}{{\mathcal{Q}}}

\usepackage{prettyref}
\newcommand{\pref}[1]{\prettyref{#1}}

\newcommand{\savehyperref}[2]{\texorpdfstring{\hyperref[#1]{#2}}{#2}}
\newrefformat{eq}{\savehyperref{#1}{Eq. \textup{(\ref*{#1})}}}
\newrefformat{eqn}{\savehyperref{#1}{Eq.~(\ref*{#1})}}
\newrefformat{lem}{\savehyperref{#1}{Lemma~\ref*{#1}}}
\newrefformat{def}{\savehyperref{#1}{Definition~\ref*{#1}}}
\newrefformat{line}{\savehyperref{#1}{Line~\ref*{#1}}}
\newrefformat{thm}{\savehyperref{#1}{Theorem~\ref*{#1}}}
\newrefformat{corr}{\savehyperref{#1}{Corollary~\ref*{#1}}}
\newrefformat{cor}{\savehyperref{#1}{Corollary~\ref*{#1}}}
\newrefformat{sec}{\savehyperref{#1}{Section~\ref*{#1}}}
\newrefformat{app}{\savehyperref{#1}{Appendix~\ref*{#1}}}
\newrefformat{assum}{\savehyperref{#1}{Assumption~\ref*{#1}}}
\newrefformat{asm}{\savehyperref{#1}{Assumption~\ref*{#1}}}
\newrefformat{ex}{\savehyperref{#1}{Example~\ref*{#1}}}
\newrefformat{fig}{\savehyperref{#1}{Figure~\ref*{#1}}}
\newrefformat{alg}{\savehyperref{#1}{Algorithm~\ref*{#1}}}
\newrefformat{rem}{\savehyperref{#1}{Remark~\ref*{#1}}}
\newrefformat{conj}{\savehyperref{#1}{Conjecture~\ref*{#1}}}
\newrefformat{prop}{\savehyperref{#1}{Proposition~\ref*{#1}}}
\newrefformat{proto}{\savehyperref{#1}{Protocol~\ref*{#1}}}
\newrefformat{prob}{\savehyperref{#1}{Problem~\ref*{#1}}}
\newrefformat{claim}{\savehyperref{#1}{Claim~\ref*{#1}}}
\newrefformat{que}{\savehyperref{#1}{Question~\ref*{#1}}}
\newrefformat{op}{\savehyperref{#1}{Open Problem~\ref*{#1}}}
\newrefformat{fn}{\savehyperref{#1}{Footnote~\ref*{#1}}}

\makeatletter \newcommand{\ForceCrefTypeInEnv}[2]{
\AddToHook{env/#1/begin}{
\let\Cref@oldlabel\label \def\label##1{\Cref@oldlabel[#2]{##1}}
}
\AddToHook{env/#1/end}{\let\label\Cref@oldlabel}} 
\makeatother
\ForceCrefTypeInEnv{lemma}{lemma} 
\ForceCrefTypeInEnv{definition}{definition} 
\ForceCrefTypeInEnv{proposition}{proposition} 
\ForceCrefTypeInEnv{claim}{claim} 
\ForceCrefTypeInEnv{corollary}{corollary} 
\ForceCrefTypeInEnv{example}{example}
\ForceCrefTypeInEnv{remark}{remark}
\ForceCrefTypeInEnv{property}{property}
\Crefname{lemma}{Lemma}{Lemmas}
\Crefname{definition}{Definition}{Definitions} 
\Crefname{proposition}{Proposition}{Propositions} 
\Crefname{claim}{Claim}{Claims}
\Crefname{corollary}{Corollary}{Corollaries} 
\Crefname{example}{Example}{Examples}
\Crefname{remark}{Remark}{Remarks}
\Crefname{property}{Property}{Properties}

\crefname{ALG@line}{line}{lines}
\Crefname{ALG@line}{Line}{Lines}
\makeatletter
\newcommand{\alglinelabel}[1]{
  \addtocounter{ALG@line}{-1}
  \refstepcounter{ALG@line}
  \label{#1}
}
\makeatother

\newcommand{\mail}[1]{\href{mailto:#1}{\color{blue} #1}}
\newcommand{\blackfootnote}[1]{
  \begingroup
  \hypersetup{allcolors=black}
  \footnote{#1}
  \endgroup
}

\title{Near-Optimal Stochastic Linear Bandits with Delay}
\author{Ofir Schlisselberg\blackfootnote{Tel Aviv University; \mail{ofirs4@mail.tau.ac.il}} 
\and
Mengxiao Zhang\blackfootnote{University of Iowa; \mail{mengxiao-zhang@uiowa.edu}}
\and
Yishay Mansour\blackfootnote{Tel Aviv University and Google Research; \mail{mansour.yishay@gmail.com}}}

\begin{document}

\maketitle

\begin{abstract}
We study stochastic linear bandits with delayed feedback under several delay models and establish near-optimal regret guarantees. Our results identify when delayed linear bandits exhibit the same qualitative behavior as multi-armed bandits (MAB), and when the linear structure creates fundamentally new challenges. Specifically, (1) for \emph{loss-independent delays}, where the delay does not depend on the realized loss (but potentially depends on the arm), we show that delays incur only an additive regret penalty. Under stochastic delays, this penalty scales with the expected delay, while under adversarial delays, it scales with the maximum number of outstanding observations. Notably, both delay penalties are dimension-free, improving upon the state-of-the-art results; (2) for \emph{loss-dependent delays}, we show that linear bandits are substantially harder than MAB: unlike in MAB, we prove matching (up to log factors) upper and lower bounds in linear bandits, whose  delay penalty depends on the square root of the  dimension. (3) for the \emph{delay-as-payoff model}, a special case of loss-dependent delay, we show that the optimal MAB guarantee, which depends only on the delay of the optimal arm, is also unattainable in linear bandits. Together, these results provide a sharp characterization of how delayed feedback interacts with linear generalization.

\end{abstract}
\section{Introduction}

Sequential decision-making under uncertainty is a central problem in machine learning, with applications ranging from recommendation systems to clinical trials. In many real-world scenarios, however, feedback is not observed immediately: actions are taken at the current round, while their outcomes are revealed only after a delay. Such delayed feedback creates a fundamental challenge, since the learner must continue making decisions without access to the most recent observations.

The stochastic multi-armed bandit (MAB) model has served as a primary framework for understanding delayed feedback. Early works considered constant delays~\citep{dudik2011efficient}, and \citet{joulani2013online} later provided a general reduction for both stochastic and adversarial delays to the case without delay. However, these reductions showed that the regret incurred in the case with delayed feedback will have a delay-dependent term that is also dependent on the number of arms $K$: the delay penalty scales as $K\E[d]$ under stochastic delays, where $\E[d]$ is the expected delay, and as $K\sigma_{\max}$ under adversarial delays, where $\sigma_{\max}$ is the maximum number of missing observations over the horizon. 
Subsequent works showed that this dependence on $K$ is in fact not intrinsic. In particular, \citet{lancewicki2021stochastic} obtained optimal stochastic-delay bounds without the extra $K$ factor in the delay-dependent term, and \citet{schlisselberg2025improved} extended this improvement to adversarial delays. Thus, for standard delayed MAB, the effect of delay is now well understood: delay contributes an additive term that matches the intrinsic delay complexity of the problem.

More involved delay structures are studied for MAB subsequently. 
\citet{manegueu2020stochastic} first studied the case where the delays are arm-dependent, whose results are later improved by \citet{lancewicki2021stochastic}. The latter work further considered delays that depend on the realized loss, showing that the regret overhead due to delay remains bounded by $\max_a \E[d(a)]$.\footnote{Their result is in fact more involved and depends on the quantiles of the delay distribution. For clarity, we present a simplified form of the bound, which we derive for a slight variant of their algorithm in Appendix~\ref{sec:mab-loss-dependent}.} A particularly important special case is the \emph{delay-as-payoff} setting, introduced by \citet{pmlr-v238-tang24c}, in which the delay is proportional to the loss. In this setting, \citet{Schlisselberg_Cohen_Lancewicki_Mansour_2025} showed that the optimal additive delay term is the mean delay of the best arm, reflecting the fact that once the best arm is identified, the learner does not need to wait for feedback from suboptimal arms.

In contrast, delayed feedback in \emph{stochastic linear bandits} remains much less understood. Linear bandits model a richer decision space in which actions are represented by an $n$-dimensional feature vectors and feedback from one action can inform the estimates of many others through an unknown linear parameter. However, unlike in MAB, where each arm can be learned essentially independently, delayed observations in linear bandits affect the estimation of a shared parameter. 
Early works on delayed linear bandits obtained regret bounds scaling with the \emph{sum} of delays over the horizon~\citep{zhou2019learning,blanchet2024delay}, which can be prohibitively large. The first result avoiding this dependence is \citet{howson2023delayed}, which obtained an additive delay term of order $n^{3/2}\E[d]$. While this improves over the sum-of-delay bounds, it still leaves unresolved whether the dimension dependence is an artifact of the analysis or an inherent feature of delayed linear bandits. More generally, the following fundamental question is open:

\smallskip
\centerline{\emph{When can the regret of delayed linear bandits match the delay dependence of MAB?}}

In this paper, we study delayed feedback in linear bandits under both loss-independent and loss-dependent delay models, and provide a unified view of when dimension plays a role.

\begin{itemize}[leftmargin=*,nosep]
    \item \textbf{Loss-independent delays.} We consider arm-dependent delays that are independent of the realized loss. For stochastic delays, we obtain a regret bound of $O\big(\sqrt{nT\log(KT)} + \max_a \E[d(a)]\big)$, where $T$ is the time horizon and $\E[d(a)]$ is the expected delay of arm $a$, matching the optimal dependence known in MAB. For adversarial delays, we prove a regret bound of $O\big(\sqrt{nT\log(KT)} + \sigma_{\max}\log(\sigma_{\max})\big)$,
    which is tight up to logarithmic factors and, to the best of our knowledge, is the first such result even in the MAB setting for arm-dependent delays. 

    \item \textbf{Loss-dependent delays.} We demonstrate that linear bandits are \emph{fundamentally harder} in this regime. While MAB regret overhead scales only with $\max_{a}\mathbb{E}[d(a)]$, we prove that linear bandits inherently incur a dimension-dependent delay penalty of $\sqrt{n}\max_{a}\mathbb{E}[d(a)]$. We establish this by providing a matching (up to logarithmic factors) upper bound and a lower bound, proving that this dependence on the dimension $n$ is unavoidable.

    \item \textbf{Delay-as-payoff.} In MAB, the delay-as-payoff setting admits an additive delay term of $d^\star$, where $d^\star$ represents the mean delay of the best arm. We show that this is \emph{not achievable} in linear bandits: even when $d^\star=0$, any algorithm must incur an additive regret of order $D$, the maximal delay. This matches the upper bound of \citet{zhang2025contextual} and establishes its optimality.
\end{itemize}

Our results reveal a sharp dichotomy. When delays are independent of the loss, linear bandits behave similarly to MAB, and optimal delay dependence can be achieved without additional dimension factors. In contrast, when delays depend on the loss, the shared structure of linear bandits fundamentally limits the learner’s ability to exploit partial feedback.

This highlights an intrinsic difference between independent and dependent delay models, and provides a complete picture of delay in linear bandits.

\begin{table}[t]
    \centering
    \caption{\textbf{Comparison of delay dependence across models.} 
Arm-independent delays are a special case of arm-dependent delays. 
We show that linear bandits match MAB guarantees for loss-independent delays, 
but exhibit fundamentally stronger dependence in loss-dependent settings. 
$n$ denotes the dimension, $\E[d(a)]$ is the expected delay of arm $a$ (and $\E[d]$ in the arm-independent case), $\sigma_{\max}$ is the maximum number of missing observations, $d^*$ is the delay of the optimal arm, and $D$ is the maximum delay. All bounds are up to absolute constants and log factors.}
\begin{tabular}[c]{ll ccc}
\multicolumn{2}{c}{\textbf{Delay model}} & \textbf{MAB} & \textbf{Linear (prior)} & \textbf{Linear (this work)} \\
\midrule
\multirow{2}{*}{Arm-inde.} 
& \small{Stochastic} & $\E[d]$ [\citenum{lancewicki2021stochastic}]  & $n^{3/2}\E[d]$ [\citenum{howson2023delayed}]  & $\E[d]$ \\

& \small{Adversarial} & $\sigma_{\max}$ [\citenum{schlisselberg2025improved}] &  --- & $\sigma_{\max}$\\
\hline
\multirow{2}{*}{Arm-dep.}
& \small{Stochastic} & $\max_a \E[d(a)]$ [\citenum{lancewicki2021stochastic}] 
& \multirow{2}{*}{---} 
& $\max_a \E[d(a)]$ \\

& \small{Adversarial} & ---& & $\sigma_{\max}$\\
\hline
Loss-dep. & \small{Stochastic} 
& $\max_a \E[d(a)]$ [\citenum{lancewicki2021stochastic}] 
& --- 
& $\sqrt{n}\max_a \E[d(a)]$ \\
\hline
\multicolumn{2}{c}{Delay-as-payoff}
& $d^*$ [\citenum{Schlisselberg_Cohen_Lancewicki_Mansour_2025}]
& $D$ (upper bound) [\citenum{zhang2025contextual}]
& $D$ (lower bound) \\
\bottomrule
\end{tabular}
\label{tab:comparison}
\end{table}
\subsection{Additional Related Work}

\textbf{Stochastic Linear Bandits.}
The stochastic linear bandit framework was first introduced by \citet{abe1999Associative}, with foundational theoretical contributions subsequently established by \citep{auer2002using, dani2008stochastic, rusmevichientong2010linearly, yadkori2011improved}. Our algorithm is built upon the phased elimination approach, a classic technique for this problem that was first adapted to the linear setting by \citet{Valko2014Spectral}. The specific algorithmic framework employed in this work follow the one presented in \citet[Chapter 22]{lattimore2020bandit}.

\textbf{Additional Work on Delayed Bandits.}
Several works examine delay under various feedback assumptions. \citep{vernade2017stochastic,vernade2020linear} study a model with Bernoulli rewards where zero outcomes are never revealed, even after the delay, which can be interpreted as a particular form of loss-dependent delay. Their results rely on assuming a bounded expected delay and that the delay distribution is fully known.

\citet{pike2018bandits} consider a more challenging scenario in which the learner observes only the total reward arriving at each round, without knowing which actions generated it. They assume that the expected delay is known and derive guarantees comparable to those of \citet{joulani2013online}.

Another line of work considers adversarial losses with delay \citep{cesa2016delay,thune2019nonstochastic,pmlr-v139-gyorgy21a,van2021nonstochastic,van2023unified}, with \citep{van2023unified} extending these ideas to linear bandits. 
Finally, \citep{masoudian2022best,masoudian2024best,schlisselberg2025improved} study delayed best-of-both-worlds guarantees, where a single algorithm adapts to both stochastic and adversarial losses.
\section{Preliminaries}\label{sec:preliminaries}
\paragraph{Learning model.} In this paper, we consider \emph{stochastic delayed linear bandits} over a finite action set $\A$ with $|\A|=K$ and an unknown loss vector $\thetastar$.\footnote{The analysis extends to continuous action sets via standard discretization arguments. Our bounds scale as $\sqrt{nT\log(K)}$, which recovers the usual $n\sqrt{T}$ dependence in the continuous setting up to logarithmic factors.} Each action $a \in \A$ is associated with a zero-mean, 1-sub-Gaussian noise distribution $\calP(a)$ such that for any $\eta \sim \calP(a)$, the bound $\abs{a^\top \thetastar + \eta} \le 1$ holds almost surely. In the case of \emph{delay-as-payoff}, following \citet{zhang2025contextual}, we further require $a^\top \thetastar + \eta \ge 0$ to ensure that all realized delays are non-negative.

The interaction protocol proceeds as follows: at each round $t\in [T]\triangleq \{1,2,\dots, T\}$:
\begin{itemize}[nosep,leftmargin=*]
    \item The learner selects an action $a_t \in \A$.
    \item The environment samples random noise $\eta_t \sim \calP(a_t)$ and generates a loss $\ell_t = a_t^\top \thetastar + \eta_t$.
    \item The environment determines a delay $d_t(a_t)$ according to the specific delay model.
    \item The learner receives the set of all losses whose feedback arrives at time $t$: $\curly{(t', \ell_{t'}): t' + d_{t'}(a_{t'}) = t}$.
\end{itemize}
The learner's objective is to minimize their (pseudo) regret, defined as the difference between the expected loss of the best fixed action and the learner's cumulative loss. Formally, we define regret over a time horizon $T$ as:
$$
    \R_T \triangleq \E\left[\sum_{t=1}^T \roundy{a_t-\astar}^\top \thetastar\right],
$$
where $\astar \in \argmin_{a \in \A} a^\top \thetastar$.

\textbf{Delay model.} In the following, we introduce several types of delay models that we consider throughout this paper. Specifically, we consider
\begin{itemize}[leftmargin=*,nosep]
    \item \textbf{Loss-independent delay}: In this setting, the delay is independent of the realized loss. We consider two primary variants:
    \begin{itemize}[nosep, leftmargin=15pt]
        \item \emph{Stochastic delay}, where delay $d_t(a)$ is drawn i.i.d. from an unknown distribution $\D(a)$. 
        \item \emph{Adversarial delay}, where the entire delay sequence $\{d_t(a)\}_{t\in[T], a \in \A}$ is determined arbitrarily by an adaptive adversary.
    \end{itemize}
    For both variants, we also investigate the \textbf{arm-independent} special case, where $\D(a) = \D$ in the stochastic setting and $d_t(a) = d_t$ for all $a \in \A$ in the adversarial setting.
    \item \textbf{Loss-dependent stochastic delay}: In this regime, the delay $d_t(a)$ is stochastically coupled with the loss $\ell_t(a)$. Specifically, the pair $(\ell_t(a), d_t(a))$ is sampled i.i.d. from an arm-specific joint distribution $\calD^{\text{joint}}(a)$. This setting encapsulates real-world scenarios where the feedback latency is inherently tied to the outcome magnitude. For example, in conversion rate optimization, verifying a high-value transaction often incurs a significantly longer latency than registering a simple click.
    \item \textbf{Delay-as-payoff}: As a special instance of loss-dependent delay, we investigate a setting where the latency is a deterministic function of the loss~\citep{schlisselberg2025improved,zhang2025contextual}. Formally, we assume $d_t(a) = D \cdot \ell_t(a)$ for some fixed scaling constant $D > 0$. This is motivated by applications such as network routing, where the goal is to minimize the latency, so the latency is both the cost and the delay in observing the information. 
    
\end{itemize}
\textbf{G-Optimal design.}
We use optimal experimental design to estimate $\thetastar$ efficiently. A design $\pi$ is a distribution over $\A$, and we define $V(\pi)=\sum_{a\in\A}\pi(a)aa^\top$ and $g(\pi)=\max_{a\in\A}\|a\|_{V(\pi)^{-1}}$. A G-optimal design~\citep[Chapter 21]{lattimore2020bandit} minimizes $g(\pi)$. The General Equivalence Theorem~\citep{Kiefer_Wolfowitz} states that the optimum satisfies $g(\pi^\star)=n$ and admits a solution with support size at most $n(n+1)/2$. Moreover, approximate designs with support $O(n\log\log n)$ and $g(\pi)\le 2n$ can be computed via Frank–Wolfe-type methods~\citep{frank_wolfe,fedrov_theory_optimal_design}. Such designs control the variance of least-squares estimators and are central to achieving optimal regret in linear bandits~\citep{bubeck2012minimaxpoliciesonlinelinear, lattimore2020bandit}. Geometrically, they correspond to the minimum-volume enclosing ellipsoid of $\mathrm{conv}(\A)$, with support points as contact points on its boundary~\citep{silvey1972discussion}.

\textbf{Additional Notation.} 
For any time $t \in [T]$, let $M_t \subseteq [t-1]$ denote the set of rounds where feedback for the chosen action has not yet been received by the start of round $t$. We further define $M_t(a) \subseteq M_t$ as the subset of rounds where action $a$ was selected, such that $M_t = \bigcup_{a \in \mathcal{A}} M_t(a)$. Let $\sigma_t(a) = |M_t(a)|$ denote the number of outstanding observations for action $a$, and let $\sigma_t = |M_t| = \sum_{a \in \mathcal{A}} \sigma_t(a)$ represent the total number of missing observations at time $t$. We define the global maximum number of outstanding observations as $\sigma_{\max} = \max_{t \in [T]} \sigma_t$. When the delay follows a distribution (either loss-dependent or independent), let $\mathcal{Q}(q; a)$ denote the $q$-quantile of the delay distribution for action $a$ and let $\mathbb{E}[d(a)]$ denote the expected delay of action $a$. In the arm-independent setting, these notations simplify to $\mathcal{Q}(q)$ and $\mathbb{E}[d]$, respectively.

\section{Loss-Independent Stochastic Delay}\label{sec:lossIndStoc}

In this section, we study delayed feedback models in which the delay process is
independent of the realized loss. The delay distribution may depend on the
action, but conditional on the chosen action, the delay is independent of the
loss observation. This setting captures an important class of delayed bandit
problems in which the learner may face heterogeneous reporting times across
actions, while still avoiding the additional statistical difficulty caused by
loss-dependent censoring.

Our algorithm, \pref{alg:loss-independent-stochastic}, extends the classical
phased elimination framework to this delayed setting. At a high level, the
algorithm proceeds in phases. In phase $\phase$, it maintains an active set
$\A_\phase$ of candidate actions and an accuracy level $\epsilon_\phase$. The
goal of the phase is to collect enough observed feedback to estimate the loss
of every active action up to accuracy $\epsilon_\phase$, eliminate actions that
are provably suboptimal, and then move to the next phase with
$\epsilon_{\phase+1}=\epsilon_\phase/2$.

\begin{algorithm}[t]
\caption{Phased Elimination w/ Loss-Independent Stochastic Delays (Formally in \Cref{alg:loss-independent-stochastic})}
\label{alg:loss-independent-informal-stochastic}
\begin{algorithmic}[1]
\State Initialize active set $\A_1 \gets \A$, accuracy $\epsilon_1 \gets \tfrac{1}{2}$
\For{phases $\phase = 1,2,\dots$}
    \State Compute a design $\pi_\phase$ over $\A_\phase$
    \State Set target samples $N_\phase = \tilde{O}\!\left(\frac{n}{\epsilon_\phase^2}\right)$
    
    \While{some arm $a$ has fewer than $N_\phase(a)=N_\phase\pi_\phase(a)$ observed samples}
        \State Sample $a_t \sim \pi_\phase$
        \State Collect any feedback that arrives
    \EndWhile

    \State Estimate $\widehat{\theta}_\phase$ from the collected samples
    \State $\displaystyle \A_{\phase+1} \gets 
    \curly{a \in \A_\phase : \inprod{\widehat{\theta}_\phase \,,\, a - b} \le 2\epsilon_\phase \quad \forall b\in \A_\phase}$ \label{line:elimination}
    \State Set $\epsilon_{\phase+1} \gets \epsilon_\phase / 2$
\EndFor
\end{algorithmic}
\end{algorithm}

The main challenge is that, under stochastic delayed feedback, the number of actions played in a phase can significantly exceed the number of observations received. In the standard non-delayed setting, this is not an issue: one can simply pull each arm $a$ in the design support deterministically to match the target number of pulls exactly. However, stochastic delays introduce a fundamental randomness that cannot be eliminated through deterministic scheduling. Since a phase can only terminate once the required observations have arrived for every action in the design support, the phase duration is governed by the tail behavior of the delays. Consequently, the phase length depends not only on the statistical complexity $N_\phase$, but also on the interplay between the design’s probability mass and the delay quantiles.

This creates a subtle technical obstacle. Standard optimal designs minimize the estimation error $g(\pi)$, but they may assign extremely small probability mass to certain actions. While this is efficient under immediate feedback, it becomes problematic with delays: if an action is sampled with very low probability, the time needed to collect enough observations can grow significantly, even if the design itself is statistically near-optimal. In this regime, the lower-order concentration terms, which are usually negligible, become the primary bottleneck because the delay process itself functions as an uncontrollable stochastic sampling process.

A contribution of our analysis is demonstrating that this obstacle can be removed without sacrificing statistical efficiency. It can be shown that any design can be converted into a balanced design where support probabilities are uniformly lower bounded, while the design objective $g(\pi)$ increases by only a constant factor; see \Cref{lem:design_with_min}. Combined with standard Frank--Wolfe constructions, this yields a design with support size $O(n \log \log n)$, minimum probability at least $1/(n \log \log n)$ on its support, and design value $g(\pi) = O(n)$; see \Cref{lem:design}. This balancing step is the primary technical ingredient that allows phased elimination to remain efficient under stochastic delays.

With this balanced design in place, the regret analysis separates into two parts. We first control the duration of each phase. Since every supported action is sampled with a non-negligible probability, after the algorithm has played the phase for approximately $(N_m + n \log(T) \log \log(n))/q$ rounds, each action has been selected approximately $qN_m(a)$ times with high probability. Consequently,
after waiting up to the
corresponding $q$-quantile of its delay distribution, a sufficient number of
feedback observations has arrived. This intuition is formalized by the following
phase-length bound.

\begin{restatable}{lemma}{stocDelay}\label{lem:stochastic_delay_phase}
    Let $T_m$ be the length of phase $m$ of \Cref{alg:loss-independent-stochastic}. Under stochastic arm-dependent delays, with probability at least $1 - \frac{1}{T}$,
\[
T_\phase = O\!\left(
\min_{q} \left\{
\frac{N_\phase + n\log(T)\log\log(n)}{q} + \max_a \calQ(q; a)
\right\}
\right).
\]
\end{restatable}

The proof is deferred to Appendix~\ref{sec:loss-indep-sto}.
The term $(N_\phase+n\log(T)\log\log(n))/q$ is the number of plays needed to
ensure that enough samples are generated and observed with high probability,
while $\max_a \calQ(q;a)$ accounts for the additional waiting time induced by
the delay distribution. The minimum over $q$ captures the optimal tradeoff
between playing longer to compensate for delayed arrivals and waiting for a
larger delay quantile.

The second part of the analysis is the standard elimination argument, adapted to
the delayed setting. At the end of phase $\phase$, the algorithm has received
$N_\phase(a)$ observations from every action $a$ in the support of the design.
By the choice of $N_\phase$ and the guarantee that $g(\pi_\phase)=\max_{a\in\A}\|a\|_{V(\pi_m)^{-1}}=O(n)$, the least
squares estimator $\widehat{\theta}_\phase$ is $\epsilon_\phase$-accurate
uniformly over all actions in $\A_\phase$ with high probability. Consequently,
the elimination rule in \Cref{line:elimination} has two properties: (i) the optimal
action is never eliminated, and (ii) every action that remains active after the phase
has suboptimality gap at most $8\epsilon_\phase$; see 
\Cref{lem:optimal_active} and \Cref{lem:active_action}. Therefore, once the algorithm reaches a phase with
accuracy $\epsilon_\phase$, all actions with gap larger than
$8\epsilon_\phase$ are removed.

Combining the phase-length bound with this elimination guarantee gives the
following regret bound.

\begin{restatable}{theorem}{lossIndStoc}\label{thm:loss-independent-stochastic}
    \Cref{alg:loss-independent-stochastic} guarantees that under stochastic arm-dependent delays, 
\[
\mathcal{R}_T = O\!\left(
\min_q \left\{
\sqrt{\frac{nT\log(KT)}{q}} + \max_a \calQ(q; a)
\right\}
\right).
\]
\end{restatable}

The proof is deferred to Appendix~\ref{sec:loss-indep-sto}. \Cref{thm:loss-independent-stochastic} shows that, up to logarithmic factors, delayed feedback affects the
regret through an additive delay-quantile term and a multiplicative
$1/\sqrt{q}$ factor in the statistical term. In particular, the algorithm
recovers the classic $\sqrt{nT}$ regret when the relevant constant quantile of the delay is $O(\sqrt{nT})$.

A direct corollary of \Cref{thm:loss-independent-stochastic} is obtained by translating the delay quantiles to mean delays.
By Markov's inequality, the median delay of action $a$ is at most
$2\mathbb{E}[d(a)]$; see \Cref{lem:markov}. Taking $q=1/2$ in the theorem gives
$$
\mathcal{R}_T = O\!\left(
\sqrt{nT\log(KT)} + \max_a \mathbb{E}[d(a)]
\right),$$
as stated formally in \Cref{cor:regret_with_expect}. Thus, for
loss-independent stochastic delays, the price of delay in linear bandits is additive in the
largest mean delay without dimensional dependency, which is the same as the one in MAB shown in~\citep{lancewicki2021stochastic}.

Finally, we note that the same algorithmic idea also extends beyond the
stochastic arm-dependent setting. In Appendix~\ref{sec:adversarial-independent},
we show that under adversarial but arm-independent delays, the same algorithm also achieves $O(\sqrt{nT\log(KT)}+\sigma_{\max})$ regret. Specifically, the cost of delayed feedback is
only an additive regret term $O(\sigma_{\max})$, which again is the same as the one shown in MAB~\citep{schlisselberg2025improved}.

\section{Loss-Independent Adversarial Delay}\label{sec:loss-independent-adv-delay}

We now extend our analysis to arm-dependent adversarial delays, which remain independent of the realized losses. While our algorithm still follows the high-level phased elimination structure used in the stochastic setting, the previous sampling strategy is no longer sufficient. In the stochastic regime, sampling arms according to a fixed design ensures that missing observations are recovered after certain periods with high probability. In contrast, when delays are chosen by an adversary, the adversary can concentrate delays on specific arms. If an arm is assigned a small probability mass (e.g., $\Theta(1/n)$), delaying $\sigma_{\max}$ of its observations may force the learner to wait an additional $\Theta(n \cdot \sigma_{\max})$ rounds to collect them. This leads to an undesirable dimension-dependent penalty that we want to avoid.

To resolve this issue, we modify how samples are collected in each phase. Instead of the probabilistic sampling used in the stochastic case, we now use a deterministic approach. By playing each arm a fixed number of times, we ensure that the adversary cannot stall the algorithm by targeting arms with low sampling probabilities. Under this approach, each phase is divided into two parts: 

\begin{itemize}[leftmargin=*,nosep]
    \item \textbf{First Part} (Line~\ref{line: 3} to Line~\ref{line: 5} of \Cref{alg:loss-independent}): we play every arm $a \in \mathcal{A}_m$ exactly $N_m(a)$ times. Since feedback is delayed, not all observations will be available immediately after these pulls are completed.
    \item \textbf{Second Part} (Line~\ref{line: iter-6} to Line~\ref{line: iter-8} of \Cref{alg:loss-independent}): the algorithm iteratively replays only the missing samples needed to complete the phase requirements.
\end{itemize}
Similar to the stochastic delay case, to bound the regret, it suffices to bound the length of each phase. Specifically, in the following key lemma, we show that when the delay is adversarial, \Cref{alg:loss-independent-adv-apx} guarantees that the $m$-th phase length is at most $O(N_m+\sigma_{\max}\log(\sigma_{\max}))$.

\begin{restatable}{lemma}{iterLength}\label{lem:phase-length}Let $T_m$ be the length of phase $m$ in \Cref{alg:loss-independent-adv-apx}. Then with probability at least $1-\frac{1}{T}$, \Cref{alg:loss-independent-adv-apx} guarantees that $T_m=O(N_m+\sigma_{\max}\log\sigma_{\max})$.\end{restatable}

FFThe proof of \pref{lem:phase-length} is based on a technical analysis of the number of times Line~\ref{line: iter-6} is executed and the number of arms played during each execution. The full proof is deferred to Appendix~\ref{app: lossIndAdv}. Based on this phase-length bound, the regret guarantee follows the same logic as the stochastic case, yielding an additive dependence on $\sigma_{\max}\log(\sigma_{\max})$.

\begin{restatable}{theorem}{lossIndAdvApx}\label{thm:loss-independent-adv-delay}\Cref{alg:loss-independent-adv-apx} guarantees that $\mathcal{R}_T = O\left(\sqrt{nT\log(KT)} + \sigma_{\max}\log(\sigma_{\max})\right)$.\end{restatable}

\Cref{thm:loss-independent-adv-delay} shows that in the adversarial delay setting, the price of delay in linear bandits is additive in the
maximum number of missing observations but still without dimensional dependency. Finally, we note that in the arm-independent stochastic setting, $\sigma_{\max}$ can be bounded in terms of the expected delay. Consequently, the same algorithm also achieves the desired guarantees for stochastic delays; see Appendix~\ref{sec:arm-ind-sto} for details.

\begin{algorithm}[t]
\caption{Phased Elimination w/ Adversarial Loss-Independent Delay (Formally in \Cref{alg:loss-independent-adv-apx})}\label{alg:loss-independent}
\begin{algorithmic}[1]
\State Initialize active set $\A_1 \gets \A$, accuracy $\epsilon_1 \gets \tfrac{1}{2}$
\For{phases $\phase = 1,2,\dots$}
    \State Compute a design $\pi_\phase$ over $\A_\phase$ \label{line: 3}
    \State Set target samples $N_\phase = \tilde{O}\!\left(\frac{n}{\epsilon_\phase^2}\right)$ 
    
    \State Play each arm $a$ approximately $N_\phase(a)=N_\phase\pi_\phase(a)$ times \label{line: 5}
    
    \While{some arm $a$ has fewer than $N_\phase(a)$ observed samples} \label{line: iter-6}
        \State Replay the missing samples of each arm
        \State Collect delayed feedback as it arrives \label{line: iter-8}
    \EndWhile

    \State Estimate $\widehat{\theta}_\phase$ from the collected samples
\State $\displaystyle \A_{\phase+1} \gets 
    \curly{a \in \A_\phase : \inprod{\widehat{\theta}_\phase \,,\, a - b} \le 2\epsilon_\phase \quad \forall b\in \A_\phase}$
    \State Set $\epsilon_{\phase+1} \gets \epsilon_\phase / 2$
\EndFor
\end{algorithmic}
\end{algorithm}

\section{Loss-Dependent Delays}\label{sec: loss-dependent-delay}

In this section, we switch our focus to the general loss-dependent delay model. In contrast to the
loss-independent setting, the delay may depend on the realized loss. This creates
a statistical difficulty that is absent from the previous cases: the set of
observations that arrive early is no longer an unbiased sample of the losses. For
example, if smaller losses tend to be observed faster (i.e. the delay is positively correlated with the loss), then a phase that simply
waits until enough feedback has arrived will over-represent such losses, leading to a biased estimate. Hence, the standard delayed phased-elimination strategy cannot be applied directly. In the following, we will first introduce our algorithm obtaining a regret bound of $O(\sqrt{n}\max_{a\in\A}\E[d(a)])$ delay dependency, showing a strictly worse dimensional dependency compared to the MAB case~\citep{lancewicki2021stochastic} where they do not have the $\sqrt{n}$ dependency. Then, we compliment our upper bound by showing a nearly matching lower bound with delay dependency $\Omega(\frac{\sqrt{n}\max_{a\in\A}[d(a)]}{\log (n\max_{a\in\A}[d(a)])})$, showing that the $\sqrt{n}$ dimensional dependency is intrinsic in linear bandits.

\subsection{Upper Bound}\label{sec:upper}

Our algorithm (\Cref{alg:loss-dependent-maintext}, formally given in \Cref{alg:loss-dependent}) still follows a phased elimination framework, but the estimation
step has to be modified to account for this censoring bias. In phase $\phase$,
we first fix the entire collection of samples that will be used for estimation:
we compute a design $\pi_\phase$ over the active set $\A_\phase$, set the target
number of samples
$N_\phase
    =
    \widetilde{O}\!\big(
        \frac{n}{\epsilon_\phase^2}
        +
        \frac{n^{3/2}}{\epsilon_\phase}
    \big)$,
and play each arm $a$ approximately
$N_\phase(a)=N_\phase\pi_\phase(a)$ times.\footnote{The $O({n^{3/2}}/{\epsilon_\phase})$ part is due to the fact that we are interested in the quintile $1-\epsilon_\phase/\sqrt{n}$, and to get samples from this quintile one needs $O(\sqrt{n}/\epsilon_\phase)$ samples per arms, one needs $O(n^{3/2}/\epsilon_\phase)$ samples overall.}

The important point is that the
estimation sample is determined before observing which feedback arrives. Once these designated samples have been collected, the learner continues playing arbitrary active arms solely to allow the feedback from the estimation pulls to arrive. The phase terminates once at least a $1 - \epsilon_\phase / \sqrt{n}$ fraction of the designated samples for every arm in the design has been observed. As we will demonstrate later, this specific fraction is the minimum required to ensure sufficient estimation accuracy for each action.

At the end of the phase, a small fraction of the designated samples may still be
missing. Since their delays may depend on their realized losses, simply ignoring
them would introduce bias. Instead, we construct directional confidence
estimates. For every candidate arm $a$, we form two arm-dependent estimators:
a lower estimator $\widehat{\theta}_{\phase,a}^{-}$ and an upper
estimator $\widehat{\theta}_{\phase,a}^{+}$ for the loss of arm $a$. We construct these two estimators for each action individually because of the dependency among actions: an estimator that overestimates the loss of one action may not overestimate some other actions.

Concretely, for a missing sample generated by arm $b$, its contribution to the
prediction of arm $a$ enters through $a^\top V_\phase^{-1}b$, where $V_\phase\triangleq\sum_{a\in\supp(\pi_\phase)}N_\phase(a)aa^\top$. Therefore, to form
the upper estimate for arm $a$, we complete this missing observation in the
direction that maximizes the predicted loss of $a$, using the sign of
$a^\top V_\phase^{-1}b$. The lower estimate uses the opposite completion. Thus,
although the learner cannot form a single unbiased least-squares estimator from
the partially observed data, it can still bracket the loss of each arm in the
specific direction needed for elimination. (See \Cref{eq:pes_opt} in \Cref{app:lossDependent}.)

\begin{algorithm}[t]
\caption{Phased Elimination w/ Loss-Dependent Delays (Formally in \Cref{alg:loss-dependent})}
\label{alg:loss-dependent-maintext}
\begin{algorithmic}[1]
\State Initialize active set $\A_1 \gets \A$, accuracy $\epsilon_1 \gets \tfrac{1}{2}$
\For{phases $\phase = 1,2,\dots$}
    \State Compute a design $\pi_\phase$ over $\A_\phase$ and
    
    set target samples $N_\phase = \tilde{O}\!\left(\frac{n}{\epsilon_\phase^2} + \frac{n^{3/2}}{\epsilon_\phase}\right)$
    
    \State Play each arm $a$ approximately $N_\phase(a)=N_\phase\pi_\phase(a)$ times
    
    \While{some arm $a$ has fewer than $(1-\frac{\epsilon_\phase}{\sqrt n})N_\phase(a)$ observed samples}
        \State Play arbitrary $a' \in \A_\phase$
        \State Collect delayed feedback from the initial samples
    \EndWhile

    \For{each arm $a \in \A_\phase$} 
        \State Construct optimistic estimator $\widehat{\theta}^+_{\phase,a}$ and
        
        pessimistic estimator $\widehat{\theta}^-_{\phase,a}$ using \Cref{eq:pes_opt}.
    \EndFor

    \State $\displaystyle \A_{\phase+1} \gets 
    \curly{a \in \A_\phase : \inprod{\widehat{\theta}^-_{\phase,b} \,,\, b} - \inprod{\widehat{\theta}^+_{\phase,a} \,,\, a} \le 6\epsilon_\phase \quad \forall b\in \A_\phase}$
    \State Set $\epsilon_{\phase+1} \gets \epsilon_\phase / 2$
\EndFor
\end{algorithmic}
\end{algorithm}

The elimination rule is a robust analogue of phased elimination: an arm $a$ is eliminated if there exists another arm $b$ whose pessimistic estimated loss is smaller than the optimistic estimated loss of $a$.

The idea of optimistic and pessimistic estimation appears in the MAB analysis of \citet{lancewicki2021stochastic}. Here, however, the challenge is to define optimism and pessimism in the linear setting, where missing observations affect each arm differently. The following lemma, proved in the appendix, shows that these estimators are both valid and sufficiently accurate
\begin{restatable}{lemma}{estimatorDependent}\label{lem:opt_pes_accuracy_helper}
For every phase $\phase$ and $a\in\A_\phase$,
$    0 \le \inprod{\widehat{\theta}^+_{\phase,a} - \widehat{\theta}_\phase,\, a} \le 2\epsilon_{\phase},$ and
$     0 \le \inprod{\widehat{\theta}_\phase - \widehat{\theta}^-_{\phase,a},\, a} \le 2\epsilon_{\phase}$, where $\widehat{\theta}_\phase$, defined in \Cref{def:estimator_loss_dep}, is the estimator if we had all $N_\phase$ samples.
\end{restatable}

 Thus, $\widehat{\theta}^+_{\phase,a}$ is optimistic and $\widehat{\theta}^-_{\phase,a}$ is pessimistic for arm $a$, and observing only a
$
1-\frac{\epsilon_\phase}{\sqrt n}
$
fraction of the samples introduces only an additional $O(\epsilon_\phase)$ error.
Since the optimistic and pessimistic errors are of order $\epsilon_\phase$, the standard phased elimination analysis again implies that every arm surviving phase $\phase$ is $O(\epsilon_\phase)$-suboptimal. 

It remains to bound the phase length $T_\phase$. Since we wait for a $1-\frac{\epsilon_\phase}{\sqrt{n}}$ fraction of the feedback to arrive, \Cref{lem:E'} shows that, with high probability, this requires an additional $\calQ(1 - \frac{\epsilon_\phase}{\sqrt{n}})$ rounds. Applying Markov’s inequality yields
$
T_\phase = O (N_\phase + \frac{\sqrt{n}\max_a \E[d(a)]}{\epsilon_\phase}).$

Since each active arm has suboptimality $O(\epsilon_\phase)$, the regret incurred in phase $\phase$ is bounded by
$
O\left(N_\phase \epsilon_\phase + \sqrt{n}\max_a \E[d(a)]\right)$.
Summing over phases gives the following bound.
\begin{restatable}{theorem}{lossDependGeneral}\label{thm:lossDependGeneral} \pref{alg:loss-dependent} guarantees that
    $$\mathcal{R}_T =
    O\!\left(
    \sqrt{nT\log(KT)}
    + \log(T)^2\log\log(n)\,n^{3/2}
    + \sqrt{n}\log(T)\max_a\E[d(a)]
    \right).$$
\end{restatable}

\subsection{Lower Bound}\label{sec:lower}
In this section, we show that the dependence on $\sqrt n$ in the delay-dependent term in our upper bound shown in \Cref{thm:lossDependGeneral} is necessary up to logarithmic factors.

\begin{restatable}{theorem}{lowerBoundGeneral}\label{thm:lowerBoundGeneral}
    For any $T,\overline{d},n \ge 32\log(\overline{d}n)$ there exists a distribution over loss-dependent delayed linear bandit instances with expected delay for all arms bounded by $\overline{d}$ such that any algorithm will suffer a regret of $\Omega(\frac{\sqrt{n}\overline{d}}{\log(\overline{d}n)})$ with a constant probability over the distribution and the randomness of the algorithm.
\end{restatable}

The proof of \Cref{thm:lowerBoundGeneral} is deferred to Appendix~\ref{app:lower}. The argument leverages a standard construction of a ``hard'' action set, establishing the existence of a collection of $K$ unit vectors in $n$-dimensional space such that all pairs of distinct actions are nearly orthogonal.

\begin{lemma}[Lemma 3.1 in \citet{lattimore2020learning}]\label{lem:hard_action_set_main}
There exists a set of actions $\{a_1,\dots,a_K\}\subset\mathbb{R}^n$ such that $\norm{a_i}_2 = 1~\forall i$ and $\abs{\inprod{a_i,a_j}} \le \sqrt{\frac{8\log(K)}{n}}~\forall i\neq j$.
\end{lemma}

To provide a proof sketch for \Cref{thm:lowerBoundGeneral}, let $q = \sqrt{\frac{8\log(n\overline{d})}{n}}$. We utilize the hard action set from \Cref{lem:hard_action_set_main} with $K = \lceil 2\overline{d}/q \rceil$ and choose $\theta^* = -a_{i^*}$ uniformly at random. Under this construction, arm $i^*$ has a loss of $-1$, while every other arm $i \neq i^*$ satisfies $\abs{\inprod{a_i, \theta^*}} \le q.$

Consequently, the mean losses of all suboptimal arms fall within a $q$-window of one another. The delay distribution leverages this $q$-level proximity to obscure the information that distinguishes the suboptimal arms. Specifically, the $q$-quantile of each suboptimal arm is delayed by $d/q$, while any feedback received immediately is distributed identically across all suboptimal arms. Therefore, for any $t < \overline{d}/q$, pulling any arm other than $a_{i^*}$ reveals no information about which arm is optimal; all suboptimal arms remain indistinguishable to the learner.

Since $i^*$ is chosen uniformly from $K = 2\overline{d}/q$ arms, any algorithm will sample the optimal arm during the first $\overline{d}/q$ rounds with probability at most $1/2$. Thus, with constant probability, the learner receives no informative feedback before time $\overline{d}/q$ and incurs a regret of $\Omega(\overline{d}/q) = \Omega\left(\frac{\sqrt{n}\overline{d}}{\log(n\overline{d})}\right)$.

\section{Delay as Payoff}\label{sec:payoff}

We now consider the \emph{delay-as-payoff} setting, a special case of delay-dependent delay where the delay is proportional to the loss, i.e., $d_t = D \cdot \ell_t$. This model introduces a strong coupling between losses and delays, and has been studied in the multi-armed bandit setting, where it is known that the regret admits an additive term of order $d^*$, where $d^*$ is the expected delay of the optimal arm. A follow-up work by~\citep{zhang2025contextual} provides an upper bound of order $D$ in the linear bandits case.

A natural question is whether this $D$ dependence is unavoidable in linear bandits or it can be improved to $d^*$ similar to the MAB case. In the following theorem, we show that in contrast to MAB, an additive dependence of order $d^*$ is not achievable in linear bandits. In fact, even when $d^* = 0$, an additive $\Omega(D)$ term is unavoidable, showing that the delay dependency obtained in~\citep{zhang2025contextual} is optimal.

\begin{restatable}{theorem}{lowerPayoff}\label{thm:lowerPayoff}
    For any $n\ge 24,K \le e^{n/100}, D \le 10K, T \ge D$, there exists a distribution over delay-as-payoff linear bandit instances with no noise and $d^* = 0$ for which any algorithm suffers regret $\Omega(D)$ with constant probability.
\end{restatable}

\begin{proofsketch}
The construction ensures that one action has zero delay, while all other actions incur delay at least $\Omega(D)$. Thus, the only way to obtain feedback early in the game is to identify and play this special action.

We use the probabilistic method to construct a large set of actions $\{a_1,\dots,a_K\} \subset \mathbb{R}^n$

and $\thetastar$ so that for a uniformly random index $i^*$,$\inprod{a_{i^*}, \thetastar} = 0$, \text{while} $\inprod{a_i, \thetastar} \ge c$ \text{for} all $i \neq i^*$ for some constant $c > 0$. This implies that the optimal action $a_{i^*}$ has delay $d^* = 0$, whereas all other actions incur delay at least $\Omega(D)$.

As a result, unless the algorithm plays $a_{i^*}$ early on, it receives no feedback for $\Omega(D)$ rounds. Since $i^*$ is chosen uniformly at random and the number of actions is large, with constant probability the optimal action is not played in the first $\Theta(D)$ rounds, leading to regret $\Omega(D)$.
\end{proofsketch}

We remark that the gap between MAB and linear bandits in the delay-as-payoff setting stems from a fundamental difference in how information is acquired.

In MAB, the $d^*$ bound arises because after roughly $d^*$ rounds, the learner receives feedback from the optimal arm. Moreover, arms that have accumulated a large amount of missing feedback can safely be inferred to be suboptimal, since each arm is explored independently.

In contrast, linear bandits do not allow for such independent exploration. Learning about a single action requires sampling other actions, as information is shared through the underlying parameter $\thetastar$. As a result, a large amount of missing feedback associated with a particular action does not imply that the action is suboptimal. It may simply reflect that the actions used to explore it yield large delays.

This coupling between actions prevents the learner from ruling out suboptimal directions based on missing feedback alone, and leads to the unavoidable $\Omega(D)$ regret even when $d^* = 0$.
\section{Discussion}

We studied delayed feedback in linear bandits under several models of delay, identifying a clear separation between settings in which linear bandits behave similarly to MAB and those in which they fundamentally differ. In particular, extending from arm-independent to arm-dependent delays preserves the qualitative dependence on delay and allows MAB-style guarantees without dimension dependence, whereas loss-dependent delays lead to a fundamentally harder problem in which dimension becomes unavoidable. An important direction for future work is to understand adversarial loss-dependent delays, a setting that remains unexplored even in MAB. More broadly, our results highlight that the impact of delay depends critically on how it interacts with the feedback structure, and it would be interesting to investigate whether a similar dichotomy arises in richer settings such as contextual bandits or MDPs, which have a different structure of information.
\section*{Acknowledgements}
OS and YM are supported by the European Research Council (ERC) under the European Union’s Horizon 2020 research and innovation program (grant agreement No. 882396), by the Israel Science Foundation and the Yandex Initiative for Machine Learning at Tel Aviv University and by a grant from the Tel Aviv University Center for AI and Data Science (TAD). OS is also supported by the TAD Excellence Program for Doctoral Students in Artificial Intelligence and Data Science from the Tel Aviv University Center for AI and Data Science (TAD) and from the Israeli Council for Higher Education (CHE) Fellowship for Outstanding PhD Students in Data Science.

\newpage
\bibliography{refs}
\bibliographystyle{plainnat}

\newpage
\appendix

\section{General Phased Elimination Lemmas}

We slightly abuse notation and write $(\cdot)_\phase$ to indicate this quantity at the end of phase $\phase$, e.g $\sigma_\phase$ is the number of outstanding observation at the end of phase $\phase$ and $\sigma_\phase(a)$ is the number of outstanding observation for action $a$ at the end of phase $\phase$.

This section is for general lemmas for \Cref{alg:loss-independent-stochastic,alg:loss-independent-adv-apx,alg:loss-dependent}. All those algorithms are based on Phased Elimination with some common notations: $N_\phase$ for the needed samples, $N_\phase(a) \triangleq \ceil{N_\phase\pi_\phase(a)}$, $V_\phase \triangleq \sum_{a\in \A}N_\phase(a)aa^\top$, and $S_\phase$ for the set of observed samples in the end of the phase $\phase$.  

The following lemma shows that similar designs provide similar $g$ values .
\begin{lemma}\label{lem:two_designs}
Let $\pi,\pi'$ be two distributions over $\calA$ such that for every $a\in\calA$:
\begin{align*}
    \pi(a) \le 2\pi'(a)
\end{align*}
Recall that $g(\pi) = \max_{a\in \A}\norm{a}_{V(\pi)^{-1}}$, with $V(\pi) = \sum_{a\in \A}\pi(a)aa^\top$. Thus:
\begin{align*}
    g(\pi') \le 2g(\pi).
\end{align*}
\end{lemma}
\begin{proof}
\begin{align*}
    2V(\pi') = \sum_a 2\pi'(a)aa^T\succeq \sum_a \pi(a)aa^T = V(\pi)    \implies V(\pi)^{-1} \succeq \roundy{2V(\pi')}^{-1}
\end{align*}

Thus, for every $a\in \A$:
\begin{align*}
    \norm{a}_{V(\pi)^{-1}} &\ge \norm{a}_{(2V(\pi'))^{-1}} = \frac{1}{2}\norm{a}_{V(\pi')^{-1}}\\
    \implies g(\pi') &= \max_a \norm{a}_{V(\pi')^{-1}} \le 2\max_a \norm{a}_{V(\pi)^{-1}} = 2g(\pi)
\end{align*}

\end{proof}

\begin{lemma}\label{lem:design_with_min}
Let $\pi$ be a design with $\abs{\supp(\pi)} = \psi$, then there is a design $\pi'$ with $\min_{a: \pi(a) > 0} \pi'(a) \ge 1/2\psi$ and $g(\pi') \le 2g(\pi)$.
\end{lemma}
\begin{proof}
Let $f$ be the following mapping:
\begin{align*}
    f(a) = \begin{cases}
        \frac{1}{\psi} & 0 <\pi(a) < \frac{1}{\psi},\\
        \pi(a) & else.
    \end{cases}
\end{align*}
Now denote $\pi'$ to be the normalized version of $f$:
\begin{align*}
    \pi(a) = \frac{f(a)}{\sum_a f(a)}.
\end{align*}

Since the support of $\pi$ is bounded by $\psi$:
\begin{align*}
    \sum_a f(a) \le \sum_a \pi(a) + \psi\frac{1}{\psi} = 2.
\end{align*}

Since $\min_{a: f(a) > 0} f(a) \ge \psi$, we have the same for $1/2\psi$ for $\pi'$. Additionally, we have that for every $a$, $\pi(a) \le 2\pi'(a)$, and thus \Cref{lem:two_designs} concludes the proof.
\end{proof}

\begin{lemma}\label{lem:design}
There is a design $\pi$ with $\min_{a:\pi(a) > 0}\pi(a) \ge \frac{1}{n\log\log(n)}$ and $g(\pi) \le 4n$.
\end{lemma}
\begin{proof}
We use the design $\pi'$ from Note 21.2.3 in \cite{lattimore2020bandit}, which has a support of $n\log\log(n)$ and $g(\pi') \le 2n$. \Cref{lem:design_with_min} concludes the construction.
\end{proof}

The following is general good event that will be used in the analysis of \Cref{alg:loss-independent-stochastic,alg:loss-independent-adv-apx,alg:loss-dependent}.
\begin{definition}\label{def:E}
Let $\calE$ be the event that, for every phase $\phase$ and action $a$, at the end of phase $\phase$:
\begin{align}\label{eq:E_eq}
    \abs{\sum_{t\in S_\phase}a^TV^{-1}_\phase a_t\eta_t} \le \sqrt{4\norm{a}_{V_\phase^{-1}}^2\log\roundy{TK}}.
\end{align}
\end{definition}

Next, we prove that this event happens with high probability.
\begin{lemma}
\begin{align*}
    \Pr[\calE] \ge 1 - \frac{1}{T}.
\end{align*}
\end{lemma}
\begin{proof}
Fix phase $\phase$ and arm $a$. Notice that the sum in LHS of \Cref{eq:E_eq} is a sum of 1-subgaussian r.v. According to Hoeffding' inequality, with probability at least $1 - \frac{1}{T^2K}$, we have
\begin{align}\label{eq:hoeffding}
    \abs{\sum_{t\in S_\phase}a^TV^{-1}_\phase a_t\eta_t} \le \sqrt{4\sum_{t\in S_\phase}\roundy{a^TV^{-1}_\phase a_t}^2\log\roundy{TK}}.
\end{align}
We note that the design of phase $m$ and the number of times each $a\in \supp(\pi_\phase)$ is present in the above summation is in the filtration of the beginning of phase $\phase$.

Taking a union bound over all $a\in\A_\phase$ and $m\in[T]$, \Cref{eq:hoeffding} is true with probability $1-\frac{1}{T}$ for all $a\in\A_\phase$ and $m\in[T]$.

To bound the inner part we need:
\begin{align}\label{eq:helper}
\begin{aligned}    
    \sum_{t\in S_\phase}\roundy{a^TV^{-1}_\phase a_t}^2 &= \sum_{t\in S_\phase}a^TV^{-1}_\phase a_ta_t^TV_\phase^{-1}a\\
    &= a^TV^{-1}_\phase \roundy{\sum_{t\in S_\phase}a_ta_t^T}V_\phase^{-1}a\\
    &= a^TV^{-1}_\phase \roundy{N_\phase(a)a_ta_t^T}V_\phase^{-1}a\\
    &= a^TV_{\phase}^{-1}a\\
    &= \norm{a}_{V_\phase^{-1}}^2,
\end{aligned}
\end{align} 
where we used the fact that in the end of the phase, $\abs{S_\phase(a)} = N_\phase(a)$. 

Incorporating \Cref{eq:helper} in \Cref{eq:hoeffding} we get:
\begin{align*}
    \abs{\sum_{t\in S_\phase}a^TV^{-1}_\phase a_t\eta_t} \le \sqrt{4\norm{a}_{V_\phase^{-1}}^2\log\roundy{TK}}.
\end{align*}
\end{proof}

The following lemma shows that when we have enough samples, in the direction of the design, we can safely bound the $g$ function.
\begin{lemma}\label{lem:norm_a_bound}
Assume that there is some $N$ such that $N_\phase(b) \ge N\cdot \pi_\phase(b)$ for all $b\in \supp(\pi_\phase)$. Then, for every $a\in \A_\phase\subseteq\mathbb{R}^n$:
\begin{align*}
    \norm{a}_{V^{-1}_\phase}^2 \le\frac{4n}{N}.
\end{align*}
\end{lemma}
\begin{proof}
\begin{align*}
    V_\phase &= \sum_a N_\phase(a)aa^T \succeq N\sum_a \pi_\phase(a)aa^T = NV(\pi) \\
    \implies V^{-1}_\phase &\preceq \frac{1}{N}V(\pi_\phase)^{-1}\\
    \implies \norm{a}_{V_\phase^{-1}} &\le \frac{1}{N}\norm{a}_{V(\pi_\phase)^{-1}} \le \frac{4n}{N}.
\end{align*}
\end{proof}

This is a core lemma in the analysis of \Cref{alg:loss-independent-stochastic,alg:loss-independent-adv-apx,alg:loss-dependent}. We show that if we have enough samples in the direction of the design, we can safely bound the estimation error. In \Cref{alg:loss-independent-stochastic,alg:loss-independent-adv-apx} the algorithm actually have the needed number of samples and thus $\hat{\theta}_\phase$ is known to the algorithm. In \Cref{alg:loss-dependent} doesn't have enough samples, so this quantity is unknown to the algorithm.
\begin{lemma}\label{lem:estimator_accuracy}
Assume $\calE$ and phase $\phase$, and let 
\begin{align*}
    \hat{\theta}_\phase \triangleq V^{-1}_\phase \sum_{a\in \supp(\pi_\phase)}a \sum_{t\in S_\phase(a)}\ell_t,
\end{align*}
where $S_\phase(a)$ is a set of time steps with i.i.d losses such that $\abs{S_\phase(a)} \ge \frac{16n\log(TK)\pi_\phase(a)}{\epsilon_\phase^2}$.

Then for every action $a\in \A_\phase$, at the end of phase $\phase$, we have
\begin{align*}
    \abs{\inprod{\widehat{\theta}_\phase - \thetastar,\, a}} \le \epsilon_\phase.
\end{align*}
\end{lemma}
\begin{proof}
We short $V \coloneqq V_\phase$, $\widehat{\theta} \coloneqq \widehat{\theta}_\phase$. Then direct calculation shows that
\begin{align*}
    \inprod{\widehat{\theta} - \thetastar,\, a} &= \sum_{t\in S_\phase(a)}a^T\roundy{V^{-1}a_t\ell_t - \thetastar}\\
    &= \sum_{t\in S_\phase(a)}a^T\roundy{V^{-1}a_t\roundy{a_t^T\theta^\star+ \eta_t} - \thetastar}\\
    &= \sum_{t\in S_\phase(a)}a^T\roundy{V^{-1}a_ta_t^T\theta^\star- \theta^\star+ V^{-1}\eta_t}\\
    &= a^T\roundy{V^{-1}\roundy{\sum_{t\in S_\phase(a)}a_ta_t^T}\theta^\star- \theta^\star+ \sum_{t\in S_\phase(a)}V^{-1}a_t\eta_t}\\
    &= \sum_{t\in S_\phase(a)}a^TV^{-1}a_t\eta_t.\\
\end{align*}
From $\calE$ and \Cref{lem:norm_a_bound}, we know that
\begin{align*}
    \abs{\inprod{\widehat{\theta} - \thetastar,\, a}} &\le \sqrt{4\norm{a}_{V^{-1}}^2\log\roundy{TK}}\\
    &\le \sqrt{\frac{16n\log\roundy{TK}\epsilon_\phase^2}{16n\log\roundy{KT}}}\\
    &= \epsilon_\phase.
\end{align*}
\end{proof}

The following lemma shows that $N_\phase$ is big enough so that the fact that $N_\phase(a)$ can be $N_\phase\pi_\phase(a) + 1$ does affect asymptotically.
\begin{lemma}\label{lem:N_m}
For every phase $\phase$, assume $N_\phase > n\log(K)$, then we have
\begin{align*}
    \sum_a N_\phase(a) \le 2N_\phase.
\end{align*}
\end{lemma}
\begin{proof}
\begin{align*}
    \sum_a N_\phase(a) &\le \sum_{a\colon \pi_\phase(a) > 0} N_\phase\pi_\phase(a) + 1\\
    &= \supp(\pi_\phase) + \sum_{a} N_\phase\pi_\phase(a)\\
    &\le n\log\log(n) + N_\phase\\
    &< 2N_\phase.
\end{align*}
The last follows from the fact that $N_\phase > n\log(K) > n\log(n) > n\log\log(n)$.
\end{proof}
\section{Omitted Details in \Cref{sec:lossIndStoc}}\label{sec:loss-indep-sto}

\subsection{Algorithm}
In this section, we provide \Cref{alg:loss-independent-stochastic}, the formal version of \Cref{alg:loss-independent} shown in the main text.
\begin{algorithm}[H]
\caption{Phased Elimination for Delayed Linear Bandits}
\label{alg:loss-independent-stochastic}
\begin{algorithmic}[1]
\Require Action set $\A \subset \mathbb{R}^d$
\State Initialize $\A_1 \gets \A$, $\epsilon_1 \gets \tfrac{1}{2}$
\For{$\phase = 1,2,\dots$}
    \State Compute a design $\pi_\phase$ over $\A_\phase$ from \Cref{lem:design}
    \State $N_\phase = \frac{16n\log\roundy{KT}}{\epsilon_\phase^2}$
    \State $N_\phase(a) \gets \ceil{N_\phase\pi_\phase(a)}$
    \State $S_\phase(a) \gets \emptyset$ for all $a \in \A_\phase$ \Comment{observations collected in phase $\phase$ for action $a$}
    
    \While{there exists $a \in \A_\phase$ such that $|S_\phase(a)| < N_\phase(a)$}
        \State Sample $a_t \sim \pi_\phase$
        \State Receive all losses that arrive at round $t$
        \For{each newly arrived loss generated in phase $\phase$ from action $a \in \A_\phase$}
            \State Add this loss to $S_\phase(a)$
        \EndFor
    \EndWhile

    \State Set $V_\phase\triangleq  \sum_{a\in \supp(\pi_\phase)} N_\phase(a)aa^T$
    \State $\widehat{\theta}_\phase \gets V_{\phase}^{-1}\sum_{a\in \supp(\pi_\phase)}a\sum_{t\in S_\phase(a)}\ell_t$
    \State Compute estimator $\widehat{\theta}_\phase$ using data from phase $\phase$

    \State $\displaystyle \A_{\phase+1} \gets 
    \curly{a \in \A_\phase : \inprod{\widehat{\theta}_\phase \,,\, a - b} \le 2\epsilon_\phase \quad \forall b\in \A_\phase}$
    
    \State $\epsilon_{\phase+1} \gets \epsilon_\phase / 2$
\EndFor
\end{algorithmic}
\end{algorithm}

\subsection{Analysis}
We first show that, as long as $\calE$ holds (\Cref{def:E}), the optimal arm isn't eliminated. This is critical when we'll analyze below the elimination time of suboptimal arms.
\begin{lemma}\label{lem:optimal_active}
Assume $\calE$ (\Cref{def:E}), then we have for every $\phase$, $a^*\in\A_\phase$.
\end{lemma}
\begin{proof}
We prove by induction. The base holds trivially since $\A_1 = \A$.
Assume by contradiction that $a^*\in \A_\phase$ but $a^*\notin \A_{\phase+1}$, namely it was eliminated in phase $\phase$. This means there is an action $b\in \A$ such that:
\begin{align*}
    \inprod{\widehat{\theta}_\phase, \, a - b}  > 2\epsilon_\phase.
\end{align*}

From \Cref{lem:estimator_accuracy}:
\begin{align*}
    2\epsilon_\phase &< \inprod{\widehat{\theta}_\phase, \,  - b}\\
    &= \inprod{\widehat{\theta}_\phase - \thetastar, \, a^*} + \inprod{\thetastar,\, a^* - b} + \inprod{\theta^\star- \widehat{\theta}_\phase, \, b}\\
    &\le \inprod{\thetastar,\, a^* - b} + 2\epsilon_\phase\\
    \implies 0 &< \inprod{\thetastar,\, a^* - b},
\end{align*}
which contradicts the fact that $a^*$ is the optimal arm.
\end{proof}

Here we use the fact that the optimal arm is never eliminated to bound when an arm is eliminated. 
\begin{lemma} \label{lem:active_action}
Assume $\calE$ (\Cref{def:E}) and action $a$ such that $a\in\A_\phase$. Then, $\Delta(a) \le 8\epsilon_\phase$.
\end{lemma}
\begin{proof}
From \Cref{lem:optimal_active} also $a^*\in\A_\phase$. If $a\in \A_\phase$, we know that:
\begin{align*}
    2\epsilon_{\phase-1} &\ge \inprod{\widehat{\theta}_{\phase-1},\, a - a^*}\\
    &= \inprod{\widehat{\theta}_{\phase-1} - \thetastar, \, a} + \inprod{\thetastar,\, a - a^*} + \inprod{\theta^\star- \widehat{\theta}_{\phase-1}, \, a^*}\\
    &\ge \inprod{\thetastar,\, a - a^*} - 2\epsilon_{\phase-1}\\
    \implies \Delta(a) &\le 4\epsilon_{\phase-1} \le 8\epsilon_\phase,
\end{align*}
where the second inequality is from \Cref{lem:estimator_accuracy}.
\end{proof}

We now construct a general lemma to bound the regret given a bound to each phase. We will use this lemma multiple times in the followings.
\begin{lemma}\label{lem:regret}
Assume that there is some $\alpha,\beta > 0$ such that, with probability at least $1-\frac{1}{T}$, $T_\phase \le \alpha N_\phase + \beta$, then:
\begin{align*}
    \R_T = O\roundy{\sqrt{\alpha nT\log(KT)} + \beta}.
\end{align*}
\end{lemma}
\begin{proof}
Assume event $\calE$ holds, which occurs with probability at least $1-\frac{1}{T}$. From \Cref{lem:active_action}, we can assume that in the $\phase$-th phase, all actions incur a bounded regret hit of $8\epsilon_\phase$. Let $\tilde{T}_\phase \triangleq \alpha N_\phase + \beta$ denote the maximal phase length. If the phase lengths were exactly $T_\phase = \tilde{T}_\phase$ and the regret hit in every step was $\epsilon_\phase$, this would only strictly increase our upper bound on the total regret. We proceed under this worst-case assumption to bound the realized regret.

Notice that a phase with $T_\phase = T$ will be the last. Thus:
\begin{align*}
    T &> \frac{16\alpha n\log(KT)}{\epsilon_\phase^2} = 16\alpha n2^{2\phase}\log(KT)\\
    \implies  \phase &< \frac{1}{2}\log\roundy{\frac{T}{16\alpha n\log(KT)}}.
\end{align*}
Denote the final phase as $\phase_{end}$.

Based on \Cref{lem:active_action} and our definition of $T_\phase$, we can bound the regret conditioned on $\calE$ holding as follows:
\begin{align*}
    &\sum_{\phase=1}^{\phase_{end}-1} 8T_\phase\epsilon_\phase \\
    &\le \sum_{\phase=1}^{\phase_{end}-1} 8(\alpha N_\phase + \beta)\epsilon_\phase\\
    &\le \sum_{\phase=1}^{\infty} 8\beta 2^{-\phase} + \sum_{\phase=1}^{\phase_{end}-1} \alpha\frac{128n\log(KT)}{\epsilon_\phase}\\
    &\le 8\beta + 128\alpha n\log(KT) \sum_{\phase=1}^{\phase_{end}-1} 2^\phase \\
    &< 8\beta + 128\alpha n\log(KT)2^{\phase_{end}}\\
    &= 8\beta + 128\alpha n\log(KT)\sqrt{\frac{T}{16\alpha n\log(KT)}}\\
    &= 8\beta + 32\sqrt{\alpha nT\log(KT)}.
\end{align*}

Finally, we compute the expected regret $\R_T$. The probability that either event $\calE$ fails or the phase bounds do not hold is at most $\frac{2}{T}$. In the worst-case failure scenario, the maximum possible cumulative regret is bounded by $T$. Therefore, the failure case contributes at most $T \times \frac{2}{T} = 2$ to the expected regret. This yields:
$$\R_T \le 8\beta + 32\sqrt{\alpha nT\log(KT)} + 2$$
This adds only a constant term to the expected regret, concluding the proof.

\end{proof}

In the following, we prove \pref{lem:stochastic_delay_phase}. For convenience, we restate the lemma as follows.
\stocDelay*

\begin{proof}
Fix some $q\in (0,1]$. After $T_\phase$ steps in this phase we have for every $a\in \A$:
\begin{align*}
    \abs{S_\phase(a)} &= \sum_{t=1}^{T_\phase}\mathds{1}\squary{a_t=a\land t+d_t \le T_\phase}\\
    \implies
    \E\squary{\abs{S_\phase(a)}} &= \sum_{t=1}^{T_\phase}\Pr\squary{a_t=a\land t+d_t \le T_\phase}\\
    &= \sum_{t=1}^{T_\phase}\Pr\squary{a_t=a}\Pr\squary{t+d_t \le T_\phase \mid a_t=a} \\
    &\ge \sum_{t=1}^{T_\phase-\calQ(q; a)}\Pr\squary{a_t=a}\Pr\squary{d_t \le \calQ(q; a)} \\
    &= \roundy{T_\phase-\calQ(q; a)}\pi_\phase(a)q.
\end{align*}

From \Cref{lem:dann}, we know that with probability at least $1-\frac{1}{T}$,
\begin{align*}
      \abs{S_\phase(a)} &\ge \frac{1}{2}\roundy{T_\phase-\calQ(q; a)}\pi_\phase(a)q - \log\roundy{T}.
\end{align*}

Thus, on $T_\phase=\frac{2}{q}\roundy{ 2N_\phase + n\log\log(n)(\log\roundy{T}+1)} + \max_a \calQ(q; a)$, for all $a\in \A$:
\begin{align}\label{eq:S_bound}
    \abs{S_\phase(a)} \ge N_\phase\pi_\phase(a) + n\log\log(n)\pi_\phase(a)(\log\roundy{T}+1)\pi_\phase(a) - \log\roundy{T} \ge  N_\phase\pi_\phase(a) + 1 \ge N_\phase(a),
\end{align}
where the second inequality is from \Cref{lem:design}. \Cref{eq:S_bound} is the condition to end the phase, and thus on $T_\phase=\frac{2}{q}\roundy{ 2N_\phase + n\log\log(n)(\log\roundy{T}+1)} + \max_a \calQ(q; a)$ the phase is over.
\end{proof}

Now we are ready to bound the total regret.
\lossIndStoc*

\begin{proof}
The proof is done by directly combining \Cref{lem:regret,lem:stochastic_delay_phase}.
\end{proof}

We now show that a simple Markov argument can give us the desired $\E[d(a)]$ dependent bound.
\begin{lemma}\label{lem:markov}
For any random variable $d$ with its corresponding qunatile $\calQ$,
\begin{align*}
    \calQ\roundy{\frac{1}{2}} \le 2\E[d].
\end{align*}
\end{lemma}
\begin{proof}
From Markov inequality:
\begin{align*}
    \frac{1}{2} &= \Pr\squary{d \ge \calQ\roundy{\frac{1}{2}}} \le \frac{\E[d]}{\calQ\roundy{\frac{1}{2}}}\\
    \implies \calQ\roundy{\frac{1}{2}} &\le 2\E[d].
\end{align*}
\end{proof}

\begin{corollary}\label{cor:regret_with_expect}
The regret of \Cref{alg:loss-independent-stochastic} is bounded by,
\begin{align*}
    \R_T = O\roundy{\sqrt{nT\log(KT)} + \max_a\E[d(a)]}.
\end{align*}
\end{corollary}
\begin{proof}
From \Cref{thm:loss-independent-stochastic}:
\begin{align*}
    \R_T &= O\roundy{min_q \curly{\sqrt{\frac{nT\log(KT)}{q}}+\max_a \calQ(q; a)}} \\
    &\le  \sqrt{2nT\log(KT)}+\max_a \calQ(\frac{1}{2}; a) \\
    &\le \sqrt{2nT\log(KT)} + \max_a \E[d(a)],
\end{align*}
where the last is from \Cref{lem:markov}.
\end{proof}

\subsection{Arm Independent Adversarial Delay}\label{sec:adversarial-independent}
In this section we show that the same algorithm (\Cref{alg:loss-independent-stochastic}) can achieve an improved regret in the arm-independent adversarial delay case.

\begin{lemma}\label{lem:adversarial2}
\Cref{alg:loss-independent-stochastic} guarantees that with probability at least $1-\frac{1}{T}$, for all $m$,
\begin{align*}
    T_\phase = O\roundy{N_\phase + n\log\roundy{T}\log\log(n) + \sigma_{\max}}.
\end{align*}
\end{lemma}
\begin{proof}
After $T_\phase$ steps, notice for $S_\phase = \bigcup_a S_\phase(a)$,   $\abs{S_\phase} \ge T_\phase - \sigma_{\max}$. We have:
\begin{align*}
    \E\squary{\abs{S_\phase(a)}} &= \E\squary{\sum_{t\in S_\phase}Pr\squary{a_t=a}}  \ge \roundy{T_\phase - \sigma_{\max}}\pi_\phase(a).
\end{align*}

From \Cref{lem:dann}, we know that with probability at least $1-\frac{1}{T}$,
\begin{align*}
      \abs{S_\phase(a)} &\ge \frac{1}{2}\roundy{T_\phase-\sigma_{\max}}\pi_\phase(a) - \log\roundy{T}.
\end{align*}

Thus, on $T_\phase=2\roundy{N_\phase + n\log\log(n)\log\roundy{T}} + \sigma_{\max}$, we have:
\begin{align*}
    \abs{S_\phase(a)} \ge N_\phase\pi_\phase(a) + n\log\log(n)\log\roundy{T}\pi_\phase(a) - \log\roundy{T} \ge  N_\phase\pi_\phase(a) = N_\phase(a),
\end{align*}
which ends the phase.
\end{proof}

\begin{theorem}\label{thm:adversarial2}
\begin{align*}
    \R_T = O\roundy{\sqrt{nT\log(KT)} + \sigma_{\max}}
\end{align*}
\end{theorem}
\begin{proof}
The proof is done by directly combining \Cref{lem:regret,lem:adversarial2}.
\end{proof}
\section{Omitted Detailed in \Cref{sec:loss-independent-adv-delay}}\label{app: lossIndAdv}
\subsection{Algorithm}
In this section, we show \pref{alg:loss-independent-adv-apx}, the formal version of \pref{alg:loss-independent}.
\begin{algorithm}[H]
\caption{Phased Elimination for Loss Independent Delay}
\label{alg:loss-independent-adv-apx}
\begin{algorithmic}[1]
\Require Action set $\A \subset \mathbb{R}^d$
\State Initialize $\A_1 \gets \A$, $\epsilon_1 \gets \tfrac{1}{2}$
\For{$\phase = 1,2,\dots$}
    \State Compute a design $\pi_\phase$ over $\A_\phase$ from \Cref{lem:design}
    \State $N_\phase = \frac{16n\log\roundy{KT}}{\epsilon_\phase^2}$
    \State $N_\phase(a) \gets \ceil{N_\phase\pi(a)}$
    \State $S_\phase(a) \gets \emptyset$ for all $a \in \A_\phase$ \Comment{observations collected in phase $\phase$ for action $a$}

    \State Play $N_\phase(a)$ times each $a\in \A$
    \While{there exists $a \in \A_\phase$ such that $|S_\phase(a)| < N_\phase(a)$}\alglinelabel{alg:iterations}
        \State Play $N_\phase(a) - |S_\phase(a)|$ times each $a$
        \State Update $S_\phase(a)$ with newly arrived loss
    \EndWhile

    \State $V_\phase = \sum_{a\in \supp(\pi_\phase)} N_\phase(a)aa^T$
    \State $\widehat{\theta}_\phase \gets V^{-1}_\phase\sum_{a\in \supp(\pi_\phase)}a\sum_{t\in S_\phase(a)}\ell_t$
    \State Compute estimator $\widehat{\theta}_\phase$ using data from phase $\phase$

    \State $\displaystyle \A_{\phase+1} \gets 
    \curly{a \in \A_\phase : \inprod{\widehat{\theta}_\phase \,,\, a - b} \le 2\epsilon_\phase \quad \forall b\in \A_\phase}$
    
    \State $\epsilon_{\phase+1} \gets \epsilon_\phase / 2$
\EndFor
\end{algorithmic}
\end{algorithm}

\subsection{Analysis}
For some phase $\phase$, we denote the start times of the while-loop of \Cref{alg:iterations} with $t_\phase^1, t_\phase^2,\dots$. Namely, the time before the $i$-th iteration of the while-loop is $t_\phase^i$. Additionally, we denote $U_t(a)$ to be $N_m(a) - \abs{S_m(a)}$ at time $t$. 

Our goal in the next few lemmas is two bound the total length of the while loop in \Cref{alg:iterations}. The following Lemma is the main ingredient in the upcoming lemmas.
\begin{lemma}\label{lem:U_bound}
Fix phase $\phase$, We have:
\begin{align*}
    i\cdot U_{t_\phase^i}(a) \le \sigma_{t_\phase^i}(a).
\end{align*}
\end{lemma}
\begin{proof}
We'll prove by induction on $i$. The base is $i = 1$ (before all iterations) - every sample in $N_\phase(a)$ that is not in $S_\phase(a)$ is a missing observation so $U_{t_\phase^1}(a) \le \sigma_{t_\phase^1}(a)$.

Assume true for some $i$, in the next iteration each $a$ was sampled $U_{{t_\phase^i}}(a)$ times. Assume samples of $a$ was returned $\nu$ times. Notice that $U_{{t_\phase^{i+1}}}(a) = \max\curly{0, U_{t_\phase^i}(a) - \nu}$. If $U_{{t_\phase^{i+1}}}(a) = 0$, it is trivial. Else:
\begin{align*}
    \sigma_{t_\phase^{i+1}}(a) &= \sigma_{t_\phase^i}(a) + U_{t_\phase^i}(a) - \nu\\
    &\ge iU_{t_\phase^i}(a) + U_{t_\phase^i}(a) - \nu\\
    &= iU_{t_\phase^i}(a) + U_{t_\phase^{i+1}}(a) \\
    &\ge (i+1) U_{t_\phase^i}(a),
\end{align*}
where the last inequality is because $U$ only decreases. 
\end{proof}

We can now bound the total number of iterations of \Cref{alg:iterations} as a simple consequence of \Cref{lem:U_bound}.
\begin{lemma}\label{lem:num_iter}
The number of iterations of \Cref{alg:iterations} in \Cref{alg:loss-independent-adv-apx} is bounded by $\sigma_{\max}$.
\end{lemma}
\begin{proof}
Let $\nu$ be the number of iterations. It means that before the $\nu$th iteration, there was some $a$ such that $U_{{t_\phase^\nu}}(a) \ge 1$. Thus:
\begin{align*}
    \sum_a U_{{t_\phase^\nu}}(a) \ge 1.
\end{align*}

From \Cref{lem:U_bound}:
\begin{align*}
    \nu &\le \nu\sum_a U_{t_\phase^\nu}(a)\\
    &\le \sum_a \sigma_{{t_\phase^\nu}}(a)\\
    &\le \sigma_{\max}.\\
\end{align*}
\end{proof}

Another simple consequence of \Cref{lem:U_bound} is the length of each iteration.
\begin{lemma}\label{lem:iter_length}
The length of iteration $i$ of \Cref{alg:iterations} in \Cref{alg:loss-independent-adv-apx} is bounded by $\frac{\sigma_{\max}}{i}$.
\end{lemma}
\begin{proof}
Since the length of the iteration is $\sum_a U_{i}(a)$, from \Cref{lem:U_bound}:
\begin{align*}
    U_{i}(a) &\le \frac{\sigma_{i-1}(a)}{i}\\
    \implies \sum_a U_{i}(a) &\le \frac{\sigma_{\max}}{i}.
\end{align*}
\end{proof}

Now that we have a bound for both the number of iterations and the length of each one, we can bound the total length using a harmonic series bound.
\iterLength*
\begin{proof}
In this proof we call the part of the algorithm that samples $N_\phase(a)$ times "first part" and the iterations afterwards "second part".

The total length of the first part is given by \Cref{lem:N_m}. The second part length, from \Cref{lem:num_iter,lem:iter_length}:
\begin{align*}
    \sum_{i=1}^{\sigma_{\max}}\frac{\sigma_{\max}}{i}  = O\roundy{\sigma_{\max}\log\roundy{\sigma_{\max}}}.
\end{align*}
\end{proof}

Finally, we prove our regret guarantee of \pref{alg:loss-independent-adv-apx} in the arm-dependent adversarial delay case, which is to combine \Cref{lem:regret,lem:phase-length}.
\lossIndAdvApx*
\begin{proof}
We use the same analysis as for \Cref{alg:loss-independent-adv-apx} since, beside the phase length, it is the same algorithm. Thus, this is proven
directly from \Cref{lem:regret,lem:phase-length}.
\end{proof}

\subsection{Arm Independent Stochastic Delay}\label{sec:arm-ind-sto}
In this subsection we show that the same \Cref{alg:loss-independent-adv-apx} can have a regret guarantee in the stochastic delay case, under the assumption that the delay is arm-independent.
\begin{lemma} \label{lem:missing to delay mean}
If the delays are stochastic and arm-independent, with probability $1 - \frac{1}{T}$:
\begin{align*}
    \sigma_{\max} \le 2\E[d] + 8\log{T}
\end{align*}
\end{lemma}
\begin{proof}
For every $t\le T$:
\begin{align*}
    \E[\sigma(t)] &= \sum_{s=1}^t \Pr[d > t-s]\\
    &= \sum_{s=1}^t \sum_{i=t-s}^{\infty} \Pr[d=i]\\
    &\le \sum_{i=1}^\infty i\Pr[d=i]\\
    &= \E[d].
\end{align*}

Fix some $t \le T$. From \Cref{lem:cons-freedman}, w.p $1 - \frac{1}{T^2}$:
\begin{align*}
    \sigma(t) \le 2\E[d] + 8\log{T}.
\end{align*}

Taking a union bound over all $t\in[T]$ concludes the proof.
\end{proof}

\begin{corollary}
If the delays are stochastic and arm-independent, \Cref{alg:loss-independent-adv-apx} promises a regret of,
\begin{align*}
    \R_T = O\roundy{\sqrt{nT\log(KT)} + \E[d]\log\roundy{\E[d]}}.
\end{align*}
\end{corollary}
\begin{proof}
Directly from \Cref{lem:missing to delay mean,thm:loss-independent-adv-delay}.
\end{proof}
\section{Omitted Details in \Cref{sec: loss-dependent-delay}}\label{app:lossDependent}
\subsection{Omitted Details in \Cref{sec:upper}}\label{app:upper}
\subsubsection{Algorithm}
In this section, we provide \pref{alg:loss-dependent}, the formal version of \pref{alg:loss-dependent-maintext}. Before that, we first introduce some notations.
\begin{definition}
For every phase $\phase$ with design $\pi_\phase$, let $S_\phase(a)$ be the set of observed pulls of arm $a$. We define the optimistic and pessimistic estimators as,
\begin{align}\label{eq:pes_opt}
\begin{aligned}
    \widehat{\theta}^+_{\phase,a} &\coloneqq V^{-1}_\phase\sum_{b\in \supp(\pi_\phase)}b\roundy{ \sigma_\phase(b)Sign\roundy{a^TV^{-1}_\phase b} +  \sum_{t\in S_\phase(b)} \ell_t},\\
    \widehat{\theta}^-_{\phase,a} &\coloneqq V^{-1}_\phase\sum_{b\in \supp(\pi_\phase)} b\roundy{\sigma_\phase(b)Sign\roundy{a^TV_\phase^{-1}b} +  \sum_{t\in S_\phase(b)} \ell_t}
\end{aligned}
\end{align}
where we recall that $V_\phase \triangleq \sum_{a\in\supp(\pi_m)}N_\phase(a)aa^\top$.
\end{definition}

\begin{algorithm}[H]
\caption{Phased Elimination for Dependent Delayed Linear Bandits}\label{alg:loss-dependent}
\begin{algorithmic}[1]
\Require Action set $\A \subset \mathbb{R}^d$
\State Initialize $\A_1 \gets \A$, $\epsilon_1 \gets \tfrac{1}{2}$
\For{$\phase = 1,2,\dots$}
    \State Compute a design $\pi_\phase$ over $\A_\phase$ from \Cref{lem:design}
    \State $N_\phase = \max\curly{\frac{48\log(T)\log\log(n)n^{3/2}}{\epsilon_\phase}, \frac{16n\log\roundy{KT}}{\epsilon_\phase^2}}$
    \State $N_\phase(a) \gets \ceil{N_\phase\piell(a)}$
    \State $S_\phase(a) \gets \emptyset$ for all $a \in \A_\phase$ \Comment{observations collected in phase $\phase$ for action $a$}

    \State Play $N_\phase(a)$ times each $a\in \A$ \alglinelabel{line:first_part}
    \While{there exists $a \in \A_\phase$ such that $|S_\phase(a)| < \roundy{1-\frac{\epsilon_\phase}{\sqrt{n}}}N_\phase(a)$}
        \State Play arbitrary $a\in\A_\phase$
        \State Receive all losses that arrive at round $t$ from the first $N_\phase$ steps
        \For{each newly arrived loss generated in \Cref{line:first_part} from action $a \in \A_\phase$}
            \State Add this loss to $S_\phase(a)$
        \EndFor
    \EndWhile

    \State Set $V_\phase = \sum_{a\in \supp(\pi_m)} N_\phase(a)aa^T$
    \For{$a\in\A_\phase$}  
        \State $\widehat{\theta}^+_{\phase,a} = V^{-1}_\phase\sum_{b\in \supp(\pi_\phase)}b\roundy{ \sigma_\phase(b)Sign\roundy{a^TV^{-1}_\phase b} +  \sum_{t\in S_\phase(b)} \ell_t}$
        \State $\widehat{\theta}^-_{\phase,a} = V^{-1}_\phase\sum_{b\in \supp(\pi_\phase)} b\roundy{\sigma_\phase(b)Sign\roundy{a^TV_\phase^{-1}b} +  \sum_{t\in S_\phase(b)} \ell_t}$
    \EndFor
    
    \State $\displaystyle \A_{\phase+1} \gets 
    \curly{a \in \A_\phase : \inprod{\widehat{\theta}^-_{\phase,a} \,,\, a} - \inprod{\widehat{\theta}^+_{\phase,b} \,,\, b} \le 6\epsilon_\phase \quad \forall b\in \A_\phase}$
    \State $\epsilon_{\phase+1} \gets \epsilon_\phase / 2$
\EndFor
\end{algorithmic}
\end{algorithm}

\subsubsection{Analysis}
In this section, we still use the event $\calE$ as a good event (\Cref{def:E}), like in the previous sections. However, we need another good event defined as follows.

\begin{definition}
Let $\calE'$ be the event that for every $\phase$,
\begin{align*}
    T_\phase \le N_\phase + \max_{a\in \A_\phase} \calQ\roundy{1-\frac{\epsilon_\phase}{3\sqrt{n}}; a}.
\end{align*}
\end{definition}

\begin{lemma}\label{lem:E'}
$\calE'$ holds with probability at least $1-\frac{1}{T}$.
\end{lemma}
\begin{proof}
Fix some $a\in \supp(\pi_\phase)$. When $T_\phase = N_\phase +  \max_{a\in \A_\phase} \calQ\roundy{1-\frac{\epsilon_\phase}{3\sqrt{n}}; a}$, all $N_\phase(a)$ pulls of it were at least $\max_{a'\in \A_\phase} \calQ\roundy{1-\frac{\epsilon_\phase}{3\sqrt{n}}; a'} \ge \calQ\roundy{1-\frac{\epsilon_\phase}{3\sqrt{n}}; a}$ step ago, we have that each of those pulls has returned w.p at least $\roundy{1-\frac{\epsilon_\phase}{3\sqrt{n}}}$. Thus, the number of missing pulls in the first $N_\phase$ pulls is a sum of Bernoulli random variables, each w.p at least $\frac{\epsilon_\phase}{3\sqrt{n}}$. 

From \Cref{lem:cons-freedman}, w.p $1-\frac{1}{T^2}$, the number of missing is bounded by $\frac{2\epsilon_\phase}{3\sqrt{n}}N_\phase(a) + 8\log(T)$. Since $\pi_\phase(a) \ge \frac{1}{2n\log\log(n)}$, we have:
\begin{align*}
    N_\phase &\ge \frac{48\log(T)\log\log(n)n^{3/2}}{\epsilon_\phase} \\
    N_\phase(a) &\ge \frac{24\log(T)\sqrt{n}}{\epsilon_\phase}\\
    \implies \frac{\epsilon_\phase}{3\sqrt{n}}N_\phase(a) &\ge 8\log(T).
\end{align*}
Thus, the number of missing pulls is bounded by $\frac{\epsilon_\phase}{\sqrt{n}} N_\phase(a)$. Taking a union bound over all $\phase$ concludes the proof w.p $1-\frac{1}{T}$.
\end{proof}

We define $\widehat{\theta}_\phase$ as a helper to bound the optimistic and pessimistic estimators. The idea is to first bound the difference between those estimators and $\widehat{\theta}_\phase$, and then use the already-constructed difference between $\widehat{\theta}_\phase$ and $\thetastar$.
\begin{definition}\label{def:estimator_loss_dep}
For any phase $\phase$ and arm $a\in\supp(\pi_\phase)$, denote $\tilde{S}_\phase(a)$ to be the set of time steps in which the first $N_\phase(a)$ pulls of arm $a$. Then, let $\widehat{\theta}_\phase$ be the estimator of $\thetastar$ in phase $\phase$. Formally,
\begin{align*}
    \widehat{\theta}_\phase \triangleq V_\phase^{-1}\sum_{a\in \supp(\pi_\phase)}a\sum_{t\in \tilde{S}_\phase(a)}\ell_t.
\end{align*}

We note that since it is an estimator constructed from $N_\phase$ i.i.d samples, \Cref{lem:estimator_accuracy} holds for it.
\end{definition}

\estimatorDependent*

\begin{proof}
We short $V \coloneqq V_\phase$, $\widehat{\theta} \coloneqq \widehat{\theta}_\phase$.
We use $M_\phase$ here to denote the missing samples from the first $N_\phase$. That is because afterwards we play arbitrary arms and the samples is discarded anyway.

For each $b\in \supp(\pi_\phase)$,
\begin{align*}
    \inprod{\hat{\theta}^+_a - \hat{\theta},\, a} &= \sum_{b}a^TV^{-1}b\roundy{\sum_{t\in M_{\phase}(b)}Sign\roundy{a^TV^{-1}b} - \ell_t}.\\
\end{align*}

In the same way,
\begin{align*}
    \inprod{\hat{\theta} - \hat{\theta}^-_a,\, a} &= \sum_{b}a^TV^{-1}b\roundy{\sum_{t\in M_\phase({b})}\ell_t - Sign\roundy{a^TV^{-1}b}}.\\
\end{align*}

We first note that since $-1 \le \ell_t \le 1$, both expressions has,
\begin{align*}
    0 &\le \inprod{\hat{\theta}^+_a - \hat{\theta},\, a} \le 2\sum_{b}\sigma_\phase(b)\abs{a^TV^{-1}b},\\
    0 &\le \inprod{\hat{\theta} - \hat{\theta}^-_a,\, a} \le 2\sum_{b}\sigma_\phase(b)\abs{a^TV^{-1}b}.
\end{align*}

To finish the proof of both expressions,
\begin{align*}
 \sum_{b}\sigma_\phase(b)\abs{a^TV^{-1}b}
    &\le \frac{\epsilon_\phase}{\sqrt{n}}\sum_{b}N_\phase(b)\abs{a^TV^{-1}b}\\
    &= \frac{\epsilon_\phase}{\sqrt{n}}\sum_{b}N_\phase(b)\sqrt{a^TV^{-1}bb^TV^{-1}a}\\
    &= \frac{\epsilon_\phase}{\sqrt{n}}\sum_{t=1}^{N_\phase}\sqrt{a^TV^{-1}b_tb_t^TV^{-1}a}\\
    &\le \frac{\epsilon_\phase}{\sqrt{n}}\sqrt{N_\phase a^TV^{-1}\roundy{\sum_{t=1}^{N_\phase}b_tb_t^T}V^{-1}a}\\
    &= \frac{\epsilon_\phase}{\sqrt{n}} \sqrt{N_\phase a^TV^{-1}VV^{-1}a}\\
    &= \frac{\epsilon_\phase}{\sqrt{n}} \sqrt{N_\phase \norm{a}_{V_{-1}}^2}\\
    &\le  \frac{\epsilon_\phase}{\sqrt{n}}\sqrt{n}\\
    &= \epsilon_\phase,
\end{align*}
where the second inequality is Cauchy-Shwartz, and third inequality is \Cref{lem:norm_a_bound}. 
\end{proof}

\begin{corollary}\label{lem:opt_pes_accuracy}
Assume $\calE$ (\Cref{def:E}), For each phase $\phase$ and arm $a\in \A_\phase$:
\begin{align*}
    -\epsilon_\phase &\le \inprod{\widehat{\theta}^+_{\phase,a} - \thetastar,\, a} \le 3\epsilon_\phase, \\
    -\epsilon_\phase &\le \inprod{\thetastar- \widehat{\theta}^-_{\phase,a},\, a} \le 3\epsilon_\phase.
\end{align*} 
\end{corollary}
\begin{proof}
Directly from \Cref{lem:opt_pes_accuracy,lem:estimator_accuracy}.
\end{proof}

We now have the desired property - our estimation is $O(\epsilon_\phase)$ accurate. We can safely prove the two important Lemmas for phased elimination: (1) the optimal arm isn't eliminated with high probability and (2) an active arm at phase $\phase$ has a suboptimality of $O(\epsilon_\phase)$.
\begin{lemma}\label{lem:optimal_active2}
Assume $\calE$, for every $\phase$, $a^*\in\A_\phase$.
\end{lemma}
\begin{proof}
We prove this by induction. The base case is proven by the fact that every $\A_1 = \A$.
Assume by contradiction that $a^*\in \A_\phase$ but $a^*\notin \A_{\phase+1}$, namely it was eliminated in phase $\phase$. This means there is an action $b\in \A$ such that:
\begin{align*}
    \inprod{\widehat{\theta}_{\phase,a^*}^- \,,\, a^*} - \inprod{\widehat{\theta}^+_{\phase,b} \,,\, b}  > 6\epsilon_\phase.
\end{align*}

From \Cref{lem:opt_pes_accuracy}:
\begin{align*}
    6\epsilon_\phase &< \inprod{\widehat{\theta}_{\phase,a^*}^- \,,\, a^*} - \inprod{\widehat{\theta}_{\phase,b}^+ \,,\, b}\\
    &= \inprod{\widehat{\theta}_{\phase,a^*}^- - \thetastar, \, a^*} + \inprod{\thetastar,\, a^* - b} + \inprod{\theta^\star- \widehat{\theta}_{\phase,b}^+, \, b}\\
    &\le \inprod{\thetastar,\, a^* - b} + 2\epsilon_\phase\\
    \implies 4\epsilon_\phase &< \inprod{\thetastar,\, a^* - b},
\end{align*}
which contradicts the fact that $a^*$ is the optimal arm.
\end{proof}

\begin{lemma} \label{lem:active_action2}
Assume $\calE$ and action $a$ such that $a\in\A_\phase$. Then, $\Delta(a) \le 24\epsilon_\phase$
\end{lemma}
\begin{proof}
From \Cref{lem:optimal_active2} also $a^*\in\A_\phase$. If $a\in \A_\phase$ it means that,
\begin{align*}
    6\epsilon_{\phase-1} &\ge \inprod{\widehat{\theta}_{\phase,a}^- \,,\, a} - \inprod{\widehat{\theta}_{\phase,a^*}^+ \,,\, a^*} \\
    &= \inprod{\widehat{\theta}_{\phase,a}^- - \thetastar, \, a} + \inprod{\thetastar,\, a - a^*} + \inprod{\theta^\star- \widehat{\theta}_{\phase,a^*}^+, \, a^*}\\
    &\ge \inprod{\thetastar,\, a - a^*} - 6\epsilon_{\phase-1}\\
    \implies \Delta(a) &\le 12\epsilon_{\phase-1} \le 24\epsilon_\phase,
\end{align*}
the second inequality is from \Cref{lem:opt_pes_accuracy}.
\end{proof}

Finally, we prove our main result \pref{thm:lossDependGeneral} for the loss-dependent delay case.
\lossDependGeneral*
\begin{proof}
Assume $\calE$ (\Cref{def:E}) and $\calE'$.

Notice that a phase with $T_\phase = T$ will be the last. Thus:
\begin{align*}
    T &\ge T_\phase \ge N_\phase \ge \frac{16n\log(KT)}{\epsilon_\phase^2}\\
    \implies \phase &< \frac{1}{2}\log\roundy{\frac{T}{16n\log(KT)}},
\end{align*}
and denote the last phase as $\phase_{end}$.

From Markov inequality,
\begin{align*}
    \calQ\roundy{1-\frac{\epsilon_\phase}{3\sqrt{n}}; a} \le \frac{3\sqrt{n}\E[d]}{\epsilon_\phase}.
\end{align*}

Thus, $\calE'$ becomes,
\begin{align*}
    T_\phase \le N_\phase + \max_{a\in \A_\phase}   \frac{3\sqrt{n}\E[d(a)]}{\epsilon_\phase}.
\end{align*}

From \Cref{lem:active_action2},
\begin{align*}
    \R_T &\le \sum_{\phase=1}^{\phase_{end}}24T_\phase\epsilon_\phase\\
    &\le \sum_{\phase=1}^{\phase_{end}}24\epsilon_\phase\roundy{N_\phase + \max_{a\in \A_\phase} \frac{3\sqrt{n}\E[d(a)]}{\epsilon_\phase}}\\
    &\le \sum_{\phase=1}^{\phase_{end}}24\epsilon_\phase\roundy{\underbrace{\frac{48\log(T)\log\log(n)n^{3/2} + \max_{a\in \A_\phase} 3\sqrt{n}\E[d(a)]}{\epsilon_\phase}}_{(i)} + \underbrace{\frac{16n\log\roundy{KT}}{\epsilon_\phase^2}}_{(ii)}}.\\
\end{align*}

We bound each term separately. To bound $(i)$, direct calculation shows that
\begin{align*}
    (i) &= O\roundy{\sum_{\phase=1}^{\phase_{end}}{\epsilon_{\phase}\frac{\log(T)\log\log(n)n^{3/2} + \max_{a\in \A_\phase} \sqrt{n}\E[d(a)]}{\epsilon_\phase}}}\\
    &\le O\roundy{\log(T)^2\log\log(n)n^{3/2} + \max_{a\in \A_\phase} \log(T)\sqrt{n}\E[d(a)]},
\end{align*}
where we use the fact that $m_{end} \le \log(T)$.

To bound $(ii)$,
\begin{align*}
    (ii) = \sum_{\phase=1}^{\phase_{end}}O\roundy{n\log\roundy{KT}2^{-\phase}}\le O\roundy{n\log\roundy{KT}\sqrt{\frac{T}{n\log(KT)}}}= O\roundy{\sqrt{nT\log(KT)}}.
\end{align*}

Alternatively, we can write,
\begin{align*}
    \R_T &\le \sum_{\phase=1}^{\phase_{end}}24T_\phase\epsilon_\phase\\
    &\le \sum_{\phase=1}^{\phase_{end}}24\epsilon_\phase\roundy{N_\phase + \max_{a\in \A_\phase} \calQ\roundy{1-\frac{\epsilon_\phase}{3\sqrt{n}}}; a}\\
    &\le \sum_{\phase=1}^{\phase_{end}}24\epsilon_\phase\roundy{\underbrace{\frac{48\log(T)\log\log(n)n^{3/2}}{\epsilon_\phase}}_{(i)} + \underbrace{\frac{16n\log\roundy{KT}}{\epsilon_\phase^2}}_{(ii)} + \underbrace{\max_{a\in \A_\phase} \calQ\roundy{1-\frac{\epsilon_\phase}{3\sqrt{n}}; a}}_{(iii)}}.\\
\end{align*}

The bound of $(i)$ and $(ii)$ is similar to above. To bound $(iii)$, from \Cref{lem:active_action2} and the monotonicity of the CDF for every $\phase$ and $a\in\A_\phase,\,a\ne a^*$
\begin{align*}
    \calQ\roundy{1 - \frac{\epsilon_\phase}{3\sqrt{n}}; a} \le \calQ\roundy{1 - \frac{\Delta(a)}{72\sqrt{n}}; a}.
\end{align*}

If $\abs{\A_\phase} = 1$, there is no regret from this phase from \Cref{lem:optimal_active2}.
If $\abs{\A_\phase} \ge 2$, it means that $\epsilon_\phase \ge 16\Delta_{\min}$, thus,
\begin{align*}
    \calQ\roundy{1 - \frac{\epsilon_\phase}{3\sqrt{n}}; a^*} \le \calQ\roundy{1 - \frac{\Delta_{\min}}{72\sqrt{n}}; a^*},
\end{align*}
which means:
\begin{align*}
    \max_{a\in \A_\phase} \calQ\roundy{1-\frac{\epsilon_\phase}{3\sqrt{n}}; a^*} \le \max_{a\ne a^*}\calQ\roundy{1-\frac{\Delta(a)}{72\sqrt{n}}; a} + \calQ\roundy{1 - \frac{\Delta_{\min}}{72\sqrt{n}}; a^*}.
\end{align*}

We now bound $(iii)$ as follows
\begin{align*}
    (iii) &= \sum_{\phase=1}^{\phase_{end}}O\roundy{\epsilon_\phase \max_{a\in \A_\phase} \calQ\roundy{1-\frac{\epsilon_\phase}{3\sqrt{n}}}; a}\\
    &\le \sum_{\phase=1}^{\phase_{end}}O\roundy{2^{-\phase}\max_{a\ne a^*}\calQ\roundy{1-\frac{\Delta(a)}{72\sqrt{n}}; a} + \calQ\roundy{1 - \frac{\Delta_{\min}}{72\sqrt{n}}; a^*}}\\
    &\le\sum_{\phase=1}^{\infty}O\roundy{2^{-\phase} \max_{a\ne a^*}\calQ\roundy{1-\frac{\Delta(a)}{72\sqrt{n}}; a} + \calQ\roundy{1 - \frac{\Delta_{\min}}{72\sqrt{n}}; a^*}}\\
    &= O\roundy{\max_{a\ne a^*}\calQ\roundy{1-\frac{\Delta(a)}{72\sqrt{n}}; a} + \calQ\roundy{1 - \frac{\Delta_{\min}}{72\sqrt{n}}; a^*}},
\end{align*}
which finishes the proof.

\end{proof}

\subsection{Omitted Details in \Cref{sec:lower}}\label{app:lower}
In this section, we provide the omitted proof for the $\tilde\Omega(\sqrt{n}d)$ lower bound in the general loss dependent delay case.
\lowerBoundGeneral*

\begin{proof}
Denote $q = \sqrt{\frac{8\log(\overline{d}n)}{n}}$. Since $n \ge 32\log(\overline{d}n)$, $q \in (0, \frac{1}{2}]$. 

We use the action set $\A$ in \Cref{lem:hard_action_set} of size $K=\ceil{2\overline{d}/q}$. Notice that:
\begin{align*}
    \log(K) \le \log\roundy{2\overline{d}\sqrt{\frac{n}{8\log(\overline{d}n)}} + 1} \le \log(\overline{d}n).
\end{align*}

The parameter $\thetastar$ is a drawn uniformly from $\A$ and multiplied by $-1$. Thus, we have in the $i^*$th instance:
\begin{align*}
    \inprod{a_{i^*}, \,\thetastar} &= -1, \\
    \abs{\inprod{a_i, \, \thetastar}} &\le \sqrt{\frac{8\log(K)}{n}} \le q\qquad \forall i\ne i^*.
\end{align*}
Denote $\mu_a \coloneqq \inprod{a_i, \, \thetastar}$.

The delay and loss distribution is:
\begin{align*}
    \Pr[\ell_t = 1, \,d_t = 0] &= \frac{1-q}{2}, \\
    \Pr[\ell_t = -1, \,d_t = 0] &= \frac{1-q}{2}, \\
    \Pr[\ell_t = 1, \,d_t = \overline{d}/q] &= \frac{q+\mu_a}{2}, \\
    \Pr[\ell_t = -1, \,d_t = \overline{d}/q] &= \frac{q-\mu_a}{2}. \\
\end{align*}
First notice that since $\abs{\mu_a} \le q  < 1$, all probabilities are legal.

Second, we make sure the expectation is correct. Indeed, for loss expectation we have:
\begin{align*}
    1 \roundy{\frac{1-q}{2} + \frac{q+\mu_a}{2}} - 1 \roundy{\frac{1-q}{2} + \frac{q-\mu_a}{2}} = \frac{1 + \mu_a}{2} - \frac{1-\mu_a}{2} = \mu_a,
\end{align*}

and for delay expectation we have:
\begin{align*}
    \frac{\overline{d}}{q}\roundy{\frac{q+\mu_a}{2} + \frac{q-\mu_a}{2}} + 0\roundy{\frac{1-q}{2} + \frac{1-q}{2}} = \overline{d}.
\end{align*}

Crucially, conditioned on $d_t=0$, for every $i\ne i^*$ the loss is $1$ w.p $1/2$ and $-1$ w.p $1/2$. Additionally, the delay has probability exactly $q$. Thus, no matter which arm was played as long as it is not $i^*$, the results is the same until $\overline{d}/q$. Therefore, until $t=\overline{d}/q$, we can think that the only received feedback is that $a_t \ne a_{i^*}$.

Fix an algorithm and its randomness, and denote $\A'$ to be the first $\overline{d}/q$ actions that are played if there is no feedback. If $i^*$ isn't there, there regret is at least $\Omega\roundy{\overline{d}/q}$. Since there are more than $2\overline{d}/q$ instances this is true w.p $\ge 1/2$. Since this is true for every randomness of the algorithm, it is true unconditionally.
\end{proof}

\section{MAB with Loss Dependent Delay}\label{sec:mab-loss-dependent}
In this appendix we create an algorithm, with the same structure as \Cref{alg:loss-dependent}, but specifically for multi armed bandit. The idea is that we use the geometry of the problem so that it suffices to have $1-\epsilon_\phase$ quantile of the information. 

The regret bound of \citet{lancewicki2021stochastic} depends on the $(1-\Delta)$ qunatile of the delay. Since this quantile isn't multiplied by $\Delta$, we can't use Markov inequaility to bound it with $\E[d]$. However, we show here that if you (1) use Phased Elimination instead of Successive Elimination and (2) do the Markov inequality in an earlier stage (not after having the final bound), we can indeed have an additive term that depends only on the expected delay.

We denote the mean loss of each arm $\mu(a)$, $\Delta(a) = \mu(a) - \mu(a^*)$.

\begin{algorithm}[H]
\caption{Phased Elimination for Dependent Delayed MAB}\label{alg:loss-dependent-mab}
\begin{algorithmic}[1]
\Require Action set $\A$
\State Initialize $\A_1 \gets \A$, $\epsilon_1 \gets \tfrac{1}{2}$
\For{$\phase = 1,2,\dots$}
    \State $N_\phase =  \frac{24\log\roundy{KT}}{\epsilon_\phase^2}$
    \State $S_\phase(a) \gets \emptyset$ for all $a \in \A_\phase$ \Comment{observations collected in phase $\phase$ for action $a$}

    \State Play $N_\phase$ times each $a\in \A$
    \While{there exists $a \in \A_\phase$ such that $|S_\phase(a)| < \roundy{1-\epsilon_\phase}N_\phase$}
        \State Play arbitrary $a\in\A_\phase$
        \State Receive all losses that arrive at round $t$ from the first $N_\phase$ steps
        \For{each newly arrived loss generated in phase $\phase$ from action $a \in \A_\phase$}
            \State Add this loss to $S_\phase(a)$
        \EndFor
    \EndWhile

    \For{$a\in \A_\phase$}
    \State $\widehat{\mu}^+(a)_\phase = \frac{1}{N_m}\roundy{\sigma_\phase(a) +  \sum_{t\in S_\phase} \ell_t}$
    \State $\widehat{\mu}^-(a)_\phase = \frac{1}{N_m(a)}\roundy{ - \sigma_\phase(a) +  \sum_{t\in S_\phase} \ell_t}$
    \EndFor
    
    \State $\displaystyle \A_{\phase+1} \gets 
    \curly{a \in \A_\phase : \widehat{\mu}^-(a) - \widehat{\mu}^+(b) \le 4\epsilon_\phase \quad \forall b\in \A_\phase}$
    \State $\epsilon_{\phase+1} \gets \epsilon_\phase / 2$
\EndFor
\end{algorithmic}
\end{algorithm}

We start by defining two good events - the first controls the length of the phase, namely how much time we should wait for that $1- \epsilon_\phase$ quantile of information, and the second controls the estimation error of each phase $\phase$.
\begin{definition}
Let $\calE'_\texttt{MAB}$ be the event that for every $\phase$:
\begin{align*}
    T_\phase \le nN_\phase + \max_{a\in \A_\phase} \calQ\roundy{1-\frac{\epsilon_\phase}{3}; a}.
\end{align*}
\end{definition}

\begin{lemma}\label{lem:E'-mab}
$\calE'_{\texttt{MAB}}$ is true with probability at least $1-\frac{1}{T}$.
\end{lemma}
\begin{proof}
Fix some $a\in \A_\phase$. Since all $N_\phase$ pulls of it were at least $\max_{a'\in \A_\phase} \calQ\roundy{1-\frac{\epsilon_\phase}{3}; a'} \ge \calQ\roundy{1-\frac{\epsilon_\phase}{3}; a}$ step ago, we have that each of those pulls has returned w.p at least $\roundy{1-\frac{\epsilon_\phase}{3}}$. Thus, the number of missing pulls is a sum of Bernoulli random variables, each w.p at least $\frac{\epsilon_\phase}{3}$. 

From \Cref{lem:cons-freedman}, w.p $1-\frac{1}{T^2}$, the number of missing is bounded by $\frac{2\epsilon_\phase}{3}N_\phase + 8\log(T)$. To show the number of missing is bounded by $\epsilon_\phase N_\phase$ we need to show that $8\log(T) \le \frac{\epsilon_\phase}{3}N_\phase$. Indeed:
\begin{align*}
    \frac{\epsilon_\phase}{3}N_\phase = \frac{\epsilon_\phase}{3}\frac{24\log\roundy{KT}}{\epsilon_\phase^2} \ge 8\log(T).
\end{align*}

Union bound for all $\phase$ concludes the proof with probability at least $1-\frac{1}{T}$.
\end{proof}

\begin{definition}\label{lem:opt_pes_accuracy_helper2}
Let $\calE_{\texttt{MAB}}$ be the event that, for every phase $\phase$ and $a\in \A_\phase$,
\begin{align*}
    -\epsilon_\phase &\le \widehat{\mu}^+_\phase(a) - \mu(a) \le 3\epsilon_\phase, \\
    -\epsilon_\phase &\le \mu(a) - \widehat{\mu}_\phase^-(a) \le 3\epsilon_\phase.
\end{align*}  
\end{definition}

\begin{lemma}
$\calE_{\texttt{MAB}}$ is true with probability at least $1-\frac{1}{T}$.
\end{lemma}
\begin{proof}
Fix $a\in \A_\phase$. Denote $\widehat{\mu}_\phase(a)$ be the average loss of each arm in the first $N_\phase$ pulls of arm $a$.

Since $-1 \le \ell_t \le 1$, both expressions above are bounded from below by $0$. For the other side:
\begin{align*}
    \widehat{\mu}_\phase^+(a) - \widehat{\mu}_\phase(a) &= \frac{1}{N_\phase}\sum_{t\in M_\phase(a)} 1 - \ell_t \le \frac{2}{N_\phase}\sigma_\phase(a)
\end{align*}

\begin{align*}
    \widehat{\mu}_\phase(a) - \widehat{\mu}_\phase^-(a) &= \frac{1}{N_\phase}\sum_{t\in M_\phase(a)} \ell_t + 1 \le \frac{2}{N_\phase}\sigma_\phase(a).
\end{align*}
By the definition of the phase $\frac{\sigma_\phase(a)}{N_\phase} \le \epsilon_\phase$.

Thus, we only need to bound $\abs{\widehat{\mu}_\phase - \mu}$ at the end of phase $\phase$. Using Hoeffding, w.p $1 - \frac{1}{T^2K}$,
\begin{align*}
   \abs{\widehat{\mu}_\phase - \mu} \le  \epsilon_\phase.
\end{align*}

Union bound on all arms and phases concludes the proof.
\end{proof}

We now do that two classical Lemmas of Successive Elimination - optimal arm isn't eliminated with high probability and non-eliminated suboptimal arms have a bounded suboptimality.
\begin{lemma}\label{lem:optimal_active3}
Assume $\calE_{\texttt{MAB}}$, for every $\phase$, $a^*\in \A_\phase$.
\end{lemma}
\begin{proof}
We prove by induction on the phases. The base is just the fact that $a^*\in \A$. Assume $a^*\in \A_{m}$ and assume by contradiction that $a^*\notin \A_{\phase+1}$. Thus, there is an action $a\in \A_\phase$ such that:
\begin{align*}
    \widehat{\mu}_\phase^-(a^*) - \widehat{\mu}_\phase^+(a) > 4\epsilon_\phase.
\end{align*}

From $\calE_{\texttt{MAB}}$:
\begin{align*}
    4\epsilon_\phase &< \widehat{\mu}_\phase^-(a^*) - \widehat{\mu}_\phase^+(a)\\
    &= (\widehat{\mu}_\phase^-(a^*) - \mu(a^*)) + (\mu(a^*) - \mu(a)) + (\mu(a) - \widehat{\mu}_\phase^+(a)) \\
    &\le \mu(a^*) - \mu(a) + 4\epsilon_\phase\\
    \implies 0 <& \mu(a^*) - \mu(a),
\end{align*}
which contradicts the fact that $a^*$ is optimal.
\end{proof}

\begin{lemma}\label{lem:active_action3}
Assume $\calE_{\texttt{MAB}}$ and arm $a$ such that $a\in \A_\phase$. Then, $\Delta(a) \le 16\epsilon_\phase$.
\end{lemma}
\begin{proof}
From \Cref{lem:optimal_active3}, also $a^*\in \A_\phase$. From the elimination rule:
\begin{align*}
    4\epsilon_{\phase-1} &\ge \widehat{\mu}_\phase^-(a) - \widehat{\mu}_\phase^+(a^*)\\
    &= (\widehat{\mu}_\phase^-(a) - \mu(a)) + (\mu(a) - \mu(a^*)) + (\mu(a^*) - \widehat{\mu}_\phase^+(a^*))\\
    &\ge \mu(a) - \mu(a^*) - 4\epsilon_{\phase - 1}\\
    \implies \Delta(a) &\le 8\epsilon_{\phase-1} \le 16\epsilon_\phase.
\end{align*}
The second inequality is from the definition of $\calE_{\texttt{MAB}}$.
\end{proof}

Finally, we show that with the above Lemmas we can bound the regret.
\begin{theorem}
The regret of \Cref{alg:loss-dependent-mab} is bounded by.
\begin{align*}
    \R_T = O\roundy{\sqrt{nT\log(KT)} + \log(T)\max_{a} \E[d(a)]}.
\end{align*}
\end{theorem}
\begin{proof}
Assume $\calE_{\texttt{MAB}}$ and $\calE'_{\texttt{MAB}}$.

Notice that a phase with $T_\phase = T$ will be the last. Thus:
\begin{align*}
    T &\ge T_\phase \ge N_\phase \ge \frac{16n\log(KT)}{\epsilon_\phase^2}\\
    \implies \phase &< \frac{1}{2}\log\roundy{\frac{T}{16n\log(KT)}}
\end{align*}
Denote the last phase as $\phase_{end}$.

From Markov inequality:
\begin{align*}
    \calQ\roundy{1-\frac{\epsilon_\phase}{3}; a} \le \frac{\E[d]}{\epsilon_\phase}
\end{align*}

Thus, $\calE'_{\texttt{MAB}}$ becomes:
\begin{align*}
    T_\phase \le N_\phase + \max_{a\in \A_\phase}   \frac{\E[d(a)]}{\epsilon_\phase}
\end{align*}

From \Cref{lem:active_action3}:
\begin{align*}
    \R_T &\le \sum_{\phase=1}^{\phase_{end}}24T_\phase\epsilon_\phase\\
    &\le \sum_{\phase=1}^{\phase_{end}}24\epsilon_\phase\roundy{N_\phase + \max_{a\in \A_\phase} \frac{\E[d(a)]}{\epsilon_\phase}}\\
    &\le \sum_{\phase=1}^{\phase_{end}}24\epsilon_\phase\roundy{\underbrace{\frac{\max_{a\in \A_\phase} \E[d(a)]}{\epsilon_\phase}}_{(i)} + \underbrace{\frac{16n\log\roundy{KT}}{\epsilon_\phase^2}}_{(ii)}}\\
\end{align*}
We finish the proof by bounding $(i)$ and $(ii)$ respectively:
\begin{align*}
    (i) &= O\roundy{\sum_{\phase=1}^{\phase_{end}}{\epsilon_{\phase}\frac{\max_{a\in \A_\phase} \E[d(a)]}{\epsilon_\phase}}}\le O\roundy{\max_{a\in \A_\phase} \log(T)\E[d(a)]},
\end{align*}
\begin{align*}
    (ii) &= \sum_{\phase=1}^{\phase_{end}}O\roundy{n\log\roundy{KT}2^{-\phase}}\le O\roundy{n\log\roundy{KT}\sqrt{\frac{T}{n\log(KT)}}}= O\roundy{\sqrt{nT\log(KT)}}.
\end{align*}

The fact that $\mathcal{E}_{\texttt{MAB}}$ and $\mathcal{E}'_{\texttt{MAB}}$ are true w.p $1-\frac{2}{T}$ concludes the proof.
\end{proof}
\section{Omitted Details in \Cref{sec:payoff}}\label{app:lowerPayoff}
In this section, we provide the omitted proof for the $\Omega(D)$ lower bound in delay-as-payoff.
\lowerPayoff*

\begin{proof}
For any subset $S\subset[n]$ with $|S|=n/2$, define the action
\[
a(S)\;:=\;\frac{1}{\sqrt{n}}\,\mathbf{1}_S\in\mathbb{R}^n.
\]
We first show that there exists a deterministic collection of $K$ sets
$S_1,\dots,S_K\subset[n]$ with $|S_i|=n/2$ such that for all ordered pairs $i\neq j$,
\begin{equation}
\label{eq:discrep_property}
|S_i\setminus S_j|\;\ge\;\frac{n}{20}.
\end{equation}
To see this, sample $S_1,\dots,S_K$ i.i.d. uniformly from $\{S\subset[n]: |S|=n/2\}$.
Fix any ordered pair $(i,j)$ with $i\neq j$ and define the random variable
$X_{ij}:=|S_i\setminus S_j|$. Conditioned on $S_j$, $X_{ij}$ is hypergeometric with mean
$\mu=\E[X_{ij}\mid S_j]=n/4$. A standard multiplicative Chernoff bound for hypergeometric
random variables gives that for any $\delta\in(0,1)$,
\[
\Pr\!\left[X_{ij}\le (1-\delta)\mu \mid S_j\right]\;\le\;\exp\!\left(-\frac{\delta^2\mu}{2}\right).
\]
Taking $\delta=4/5$ so that $(1-\delta)\mu=(1/5)\cdot (n/4)=n/20$, we obtain
\[
\Pr\!\left[X_{ij}\le \frac{n}{20}\right]\;\le\;\exp\!\left(-\frac{(4/5)^2\cdot (n/4)}{2}\right)
\;=\;\exp\!\left(-\frac{2n}{25}\right).
\]
Union bounding over all $K(K-1)\le K^2$ ordered pairs,
\[
\Pr\!\left[\exists\,i\neq j:\ |S_i\setminus S_j|\le \frac{n}{20}\right]
\;\le\;K^2 \exp\!\left(-\frac{2n}{25}\right).
\]
Using $K\le \exp(n/100)$, we have $K^2\le \exp(n/50)$, hence
\[
K^2 \exp\!\left(-\frac{2n}{25}\right)
\;\le\;\exp\!\left(\frac{n}{50}-\frac{2n}{25}\right)
\;=\;\exp\!\left(-\frac{3n}{50}\right)
\;\le\;\frac{1}{4},
\]
where the last inequality uses $n\ge 24$. Therefore, with probability at least $3/4$,
the sampled collection satisfies \eqref{eq:discrep_property}. By the probabilistic method,
there exists a deterministic choice of $S_1,\dots,S_K$ that satisfies
\eqref{eq:discrep_property}. Fix such a choice and define $\mathcal{A}=\{a(S_1),\dots,a(S_K)\}$.

The parameter $\thetastar$ is chosen as $a\roundy{\overline{S_{i^*}}}$, where $i^*$ is drawn uniformly from $[K]$ and $\overline{S_{i^*}}$ is the complementary of $S_{i^*}$. One can see that $\inprod{a_{i^*},\,\thetastar} = 0$ which means that indeed $d^*$. For any other $i\ne i^*$, we have from \eqref{eq:discrep_property} that $\inprod{a_{i},\,\thetastar} \ge 1/20$ which means that its delay is at least $D/20$. Thus, unless $a_{i^*}$ is played in the first $D/20$ actions of the algorithm, the algorithm gets no feedback.

Fix an algorithm and all its randomness, and denote $\A'$ to be the first $D/20$ actions that are played if there is no feedback. If $i^*$ does not belong to $\A'$, then the regret is at least $\Omega\roundy{D}$ since the optimal action is not chosen for $\Omega(D)$ rounds and the suboptimality gap is $\Omega(1)$. Since, there are $K\ge D/10$ actions, the above event holds with probability at least $1/2$. Since this is true for every randomness of the algorithm, it is true unconditionally.
\end{proof}
\section{Auxiliary Lemmas}\label{app:aux}
\begin{lemma}[Lemma F.4 in \cite{dann2017unifying}]
    \label{lem:dann}
     Let $\{ X_t \}_{t=1}^T$ be a sequence of Bernoulli random and a filtration $\calF_1 \subseteq \calF_2 \subseteq...\calF_T$ with $\bbP(X_t = 1\mid \calF_t) = P_t$, $P_t$ is $\calF_{t}$-measurable and $X_t$ is $\calF_{t+1}$-measurable. Then, for all $t\in [T]$ simultaneously, with probability $1-\delta$,
     \[
        \sum_{k=1}^t X_k \geq \frac{1}{2}\sum_{k=1}^t P_k -\log \frac{1}{\delta}.
     \]
\end{lemma}

\begin{lemma}[Consequence of Freedman’s Inequality, e.g., Lemma 27 in \citet{pmlr-v139-efroni21a}]
    \label{lem:cons-freedman}
     Let $\{ X_t \}_{t\geq 1}$ be a sequence of random variables, supported in $[-R,R]$, and adapted to a filtration $\calF_1 \subseteq \calF_2 \subseteq...\calF_T$. For any $T$, with probability $1-\delta$,
     \[
        \sum_{t=1}^T X_t \leq 2 \sum_{t=1}^T\bbE[X_t \mid \calF_t] + 4R \log{\frac{1}{\delta}}.
     \]
\end{lemma}

\begin{lemma}[Lemma 3.1 in \cite{lattimore2020learning}]\label{lem:hard_action_set}
There exists a set of actions $\curly{a_1, a_2, \dots, a_K} \subset \rr^n$ such that:
\begin{align*}
    \norm{a_i}_2 & =1,~~~\forall i\in[K],\\
    \abs{\inprod{a_i, a_j}} &\le \sqrt{\frac{8\log(K)}{n}},~~~\forall j\ne i.
\end{align*}
\end{lemma}

\end{document}